\documentclass[twoside]{article}
\pdfoutput=1

\usepackage[accepted]{aistats2017}
\usepackage[labelfont=bf]{caption}
\usepackage[mathscr]{eucal}
\usepackage[noend]{algpseudocode}
\usepackage[numbers]{natbib}
\usepackage[subrefformat=parens]{subcaption}
\usepackage[T1]{fontenc}
\usepackage[toc,page]{appendix}
\usepackage[usenames,dvipsnames,svgnames]{xcolor}

\usepackage{algorithmicx}
\usepackage{algorithm}
\usepackage{amsfonts}
\usepackage{amsmath}
\usepackage{amssymb}
\usepackage{amsthm}
\usepackage{array}
\usepackage{bm}
\usepackage{booktabs}
\usepackage{colortbl}
\usepackage{enumitem}
\usepackage{fancyvrb}
\usepackage{graphicx}
\usepackage{listings}
\usepackage{placeins}
\usepackage{standalone}
\usepackage{tabulary}
\usepackage{textcomp}
\usepackage{tikz}
\usepackage{url}
\usepackage{varwidth}

\usetikzlibrary{positioning}
\usetikzlibrary{calc}
\usetikzlibrary{shadings}
\usetikzlibrary{shadows}
\usetikzlibrary{shapes.arrows}

\DeclareCaptionFormat{algor}{%
  \hrulefill\par\offinterlineskip\vskip1pt%
  \textbf{#1#2}#3\offinterlineskip\hrulefill}
\DeclareCaptionStyle{algori}{singlelinecheck=off,format=algor,labelsep=space}
\newcounter{parentalgorithm}

\makeatletter
\newenvironment{subalgorithms}{%
  \refstepcounter{algorithm}%
  \protected@edef\theparentalgorithm{\thealgorithm}%
  \setcounter{parentalgorithm}{\value{algorithm}}%
  \setcounter{algorithm}{0}%
  \def\thealgorithm{\theparentalgorithm\alph{algorithm}}%
  \ignorespaces
}{%
  \setcounter{algorithm}{\value{parentalgorithm}}%
  \ignorespacesafterend
}
\makeatother

\newcommand{\bgamma}{\bm\gamma}
\newcommand{\blambda}{\bm\lambda}
\newcommand{\bphi}{\bm\phi}

\newcommand{\D}{\mathscr{D}}

\newcommand{\G}{\mathcal{G}}
\newcommand{\I}{\mathcal{I}}

\newcommand{\mtA}{\mathcal{A}}
\newcommand{\mtB}{\mathcal{B}}
\newcommand{\mtC}{\mathcal{C}}
\newcommand{\mtE}{\mathcal{E}}

\newcommand{\mtQ}{\mathcal{Q}}
\newcommand{\mtV}{\mathcal{V}}

\newcommand{\x}{\bm{x}}
\newcommand{\hx}{\hat{\bm{x}}}

\newcommand{\z}{\bm{z}}
\newcommand{\bv}{\bm{v}}

\newcommand{\pt}{p_{\G}}
\newcommand{\ptv}{p_{\G_\mathcal{V}}}

\newcommand{\indep}{\perp\!\!\!\perp}

\newcommand{\set}[1]{\left\{{#1}\right\}}

\newtheorem{lemma}{Lemma}

\algnewcommand{\LineComment}[1]{\State \(\triangleright\) #1}

\makeatletter
\renewcommand{\paragraph}{%
  \@startsection{paragraph}{4}%
  {\z@}{0ex \@plus 1ex \@minus .2ex}{-1em}%
  {\normalfont\normalsize\bfseries}%
}
\makeatother

\newcommand{\red}[1]{\textcolor{BrickRed}{#1}}

\makeatletter
\let\OldStatex\Statex
\renewcommand{\Statex}[1][3]{%
  \setlength\@tempdima{\algorithmicindent}%
  \OldStatex\hskip\dimexpr#1\@tempdima\relax}
\makeatother

\begin{document}

%

%


\runningtitle{Detecting Dependencies in Sparse, Multivariate Databases}
\runningauthor{Saad and Mansinghka}

\twocolumn[
  \aistatstitle{Detecting Dependencies in Sparse, Multivariate Databases
  Using Probabilistic Programming and Non-parametric Bayes}

  \aistatsauthor{%
    Feras Saad \And
    Vikash Mansinghka}

  \aistatsaddress{%
    Probabilistic Computing Project\\
    Massachusetts Institute of Technology \And
    Probabilistic Computing Project\\
    Massachusetts Institute of Technology}
]


\begin{abstract}
Datasets with hundreds of variables and many missing values are commonplace.
In this setting, it is both statistically and computationally challenging to
detect true predictive relationships between variables and also to suppress
false positives.
This paper proposes an approach that combines probabilistic programming,
information theory, and non-parametric Bayes.
It shows how to use Bayesian non-parametric modeling to (i) build an ensemble of
joint probability models for all the variables; (ii) efficiently detect marginal
independencies; and (iii) estimate the conditional mutual information between
arbitrary subsets of variables, subject to a broad class of constraints.
Users can access these capabilities using BayesDB, a probabilistic programming
platform for probabilistic data analysis, by writing queries in a simple, %
SQL-like language.
This paper demonstrates empirically that the method can (i) detect %
context-specific (in)dependencies on challenging synthetic problems and (ii)
yield improved sensitivity and specificity over baselines from statistics and
machine learning, on a real-world database of over 300 sparsely observed
indicators of macroeconomic development and public health.
\end{abstract}


\section{Introduction}
\label{sec:introduction}

Sparse databases with hundreds of variables are commonplace.
In these settings, it can be both statistically and computationally challenging
to detect predictive relationships between variables \cite{data2013}.
First, the data may be incomplete and require cleaning and imputation before
pairwise statistics can be calculated.
Second, parametric modeling assumptions that underlie standard hypothesis
testing techniques may not be appropriate due to nonlinear, multivariate, and/or
heteroskedastic relationships.
Third, as the number of variables grows, it becomes harder to detect true
relationships while suppressing false positives.
Many approaches have been proposed (see \cite[Table 1]{lopez2013} for a
summary), but they each exhibit limitations in practice.
For example, some only apply to fully-observed real-valued data, and most do not
produce probabilistically coherent measures of uncertainty.
This paper proposes an approach to dependence detection that combines
probabilistic programming, information theory, and non-parametric Bayes.
The end-to-end approach is summarized in Figure~\ref{fig:workflow}.
Queries about the conditional mutual information (CMI) between variables of
interest are expressed using the Bayesian Query Language \cite{mansinghka2015},
an SQL-like probabilistic programming language.
Approximate inference with CrossCat \cite{mansinghka2016} produces an ensemble
of joint probability models, which are analyzed for structural (in)dependencies.
For model structures in which dependence cannot be ruled out, the CMI is
estimated via Monte Carlo integration.

In principle, this approach has significant advantages.
First, the method is scalable to high-dimensional data: it can be used for
exploratory analysis without requiring expensive CMI estimation for all pairs of
variables.
Second, it applies to heterogeneously typed, incomplete datasets with minimal
\mbox{pre-processing} \cite{mansinghka2016}.
Third, the non-parametric Bayesian joint density estimator used to form CMI
estimates can model a broad class of data patterns, without overfitting to
highly irregular data.
This paper shows that the proposed approach is effective on a
\mbox{real-world} database with hundreds of variables and a missing data
rate of $\sim$35\%, detecting common-sense predictive relationships that are
missed by baseline methods while suppressing spurious relationships that
baselines purport to detect.



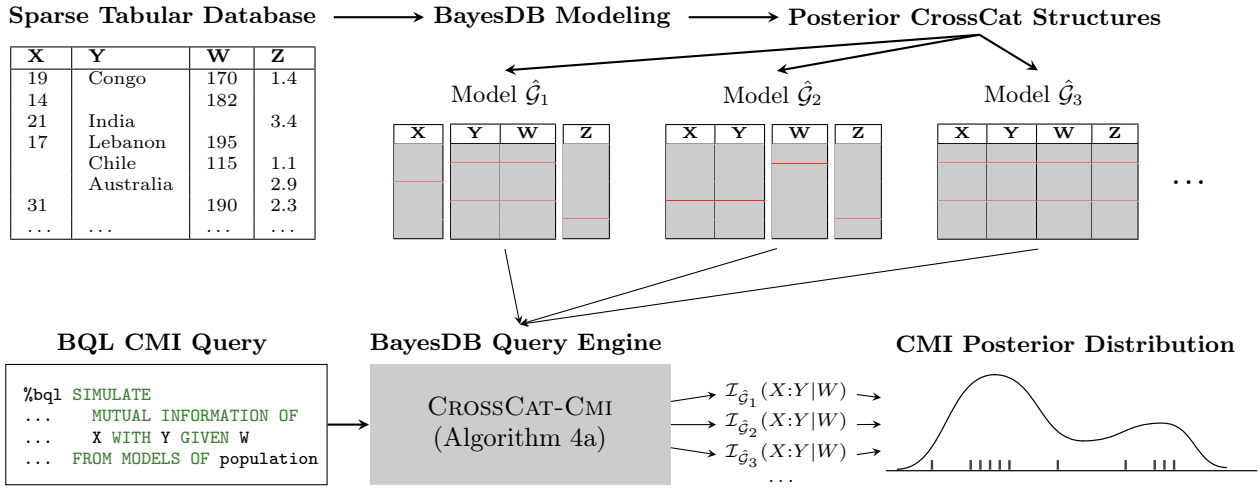
\begin{figure*}[t]

\begin{tikzpicture}


\node[
    rectangle,
    draw = none,
    align = center
] (data-table-header) {\footnotesize
    \textbf{Sparse Tabular Database}
};

\node[
  rectangle,
  draw = none,
  align = center,
  below = 0cm of data-table-header
] (data-table){
  \scriptsize
  \begin{tabular}{|l|l|l|l|}
  \hline
  \textbf{X} & \textbf{Y} & \textbf{W} & \textbf{Z} \\
  \hline
  19        & Congo         & 170       & 1.4    \\
  14        &               & 182       &        \\
  21        & India         &           & 3.4    \\
  17        & Lebanon       & 195       &        \\
            & Chile         & 115       & 1.1    \\
            & Australia     &           & 2.9    \\
  31        &               & 190       & 2.3    \\
  \dots     & \dots         & \dots     & \dots  \\
  \hline
  \end{tabular}
};


\node[
    rectangle,
    draw = none,
    align = center,
    right = 1.2 cm of data-table-header
] (bayesdb-modeling) {\footnotesize
    \textbf{BayesDB Modeling}
};

\draw [-stealth,thick] (data-table-header) -- (bayesdb-modeling);


\node[
    rectangle,
    draw = none,
    align = center,
    right = 1.2cm of bayesdb-modeling
] (crosscat-header) {\footnotesize\bf
    Posterior CrossCat Structures
};

\draw [-stealth,thick] (bayesdb-modeling) -- (crosscat-header);

\node[
  rectangle,
  draw = none,
  align = center,
  right = 4.5cm of data-table.south,
  anchor = south
] (crosscat-structure-1){
  \tiny

  \begin{tabular}{|>{\columncolor{gray!40}}c|}
  \hline
  \multicolumn{1}{|c|}{\textbf{X}}\\
  \hline
  \\
  \\
  \arrayrulecolor{red}\hline
  \\
  \\
  \\
  \arrayrulecolor{black}\hline
  \end{tabular}

  \begin{tabular}{|>{\columncolor{gray!40}}c|>{\columncolor{gray!40}}c|}
  \hline
  \multicolumn{1}{|c|}{\textbf{Y}} & \multicolumn{1}{c|}{\textbf{W}}\\
  \hline
  & \\
  \arrayrulecolor{red}\hline
  & \\
  & \\
  \arrayrulecolor{red}\hline
  & \\
  & \\
  \arrayrulecolor{black}\hline
  \end{tabular}

  \begin{tabular}{|>{\columncolor{gray!40}}c|}
  \hline
  \multicolumn{1}{|c|}{\textbf{Z}}\\
  \hline
  \\
  \\
  \\
  \\
  \arrayrulecolor{red}\hline
  \\
  \arrayrulecolor{black}\hline
  \end{tabular}
};

\node[
  rectangle,
  draw = none,
  align = center,
  right = .4cm of crosscat-structure-1
] (crosscat-structure-2){
  \tiny

  \begin{tabular}{|>{\columncolor{gray!40}}c|>{\columncolor{gray!40}}c|}
  \hline
  \multicolumn{1}{|c|}{\textbf{X}} & \multicolumn{1}{c|}{\textbf{Y}}\\
  \hline
  & \\
  & \\
  & \\
  \arrayrulecolor{red}\hline
  & \\
  & \\
  \arrayrulecolor{black}\hline
  \end{tabular}

  \begin{tabular}{|>{\columncolor{gray!40}}c|}
  \hline
  \multicolumn{1}{|c|}{\textbf{W}}\\
  \hline
  \\
  \arrayrulecolor{red}\hline
  \\
  \\
  \\
  \\
  \arrayrulecolor{black}\hline
  \end{tabular}

  \begin{tabular}{|>{\columncolor{gray!40}}c|}
  \hline
  \multicolumn{1}{|c|}{\textbf{Z}}\\
  \hline
  \\
  \\
  \\
  \\
  \arrayrulecolor{red}\hline
  \\
  \arrayrulecolor{black}\hline
  \end{tabular}
};

\node[
  rectangle,
  draw = none,
  align = center,
  right = .4cm of crosscat-structure-2
] (crosscat-structure-3){
  \tiny
  \begin{tabular}{%
    |>{\columncolor{gray!40}}c%
    |>{\columncolor{gray!40}}c%
    |>{\columncolor{gray!40}}c%
    |>{\columncolor{gray!40}}c|}
  \hline
  \multicolumn{1}{|c|}{\textbf{X}}
    & \multicolumn{1}{c|}{\textbf{Y}}
    & \multicolumn{1}{c|}{\textbf{W}}
    & \multicolumn{1}{c|}{\textbf{Z}} \\
  \hline
  & & & \\
  \arrayrulecolor{red}\hline
  & & & \\
  & & & \\
  \arrayrulecolor{red}\hline
  & & & \\
  & & & \\
  \arrayrulecolor{black}\hline
  \end{tabular}%
};


\node[
    rectangle,
    align = center,
    above = 0 of crosscat-structure-1.90,
] (crosscat-structure-1-header) {\footnotesize
    Model $\hat{\G}_1$
};

\node[
    rectangle,
    align = center,
    above = 0 of crosscat-structure-2.90,
] (crosscat-structure-2-header) {\footnotesize
    Model $\hat{\G}_2$
};

\node[
    rectangle,
    align = center,
    above = 0 of crosscat-structure-3.90,
] (crosscat-structure-3-header) {\footnotesize
    Model $\hat{\G}_3$
};

\node[
    rectangle,
    align = center,
    right = 0.2 of crosscat-structure-3,
] (crosscat-structure-dots-header) {\footnotesize
    \textbf{\dots}
};

\draw[-stealth,thick]
    (crosscat-header.south) -- (crosscat-structure-1-header.north);

\draw[-stealth,thick]
    (crosscat-header.south) -- (crosscat-structure-2-header.north);

\draw[-stealth,thick]
    (crosscat-header.south) -- (crosscat-structure-3-header.north);


\node[
    rectangle,
    draw = none,
    align = center,
    below = 1cm of data-table.south
] (bql-query-header) {\footnotesize
    \textbf{BQL CMI Query}
};

\node[
    rectangle,
    draw = black,
    align = center,
    minimum height = 1.6cm,
    below = 0cm of bql-query-header.south
] (bql-query) {
    \lstset{
      basicstyle=\ttfamily\scriptsize,
      columns=fullflexible,
      keepspaces=true,
      alsoletter={\.,\%},
      morekeywords=[1]{SIMULATE, MUTUAL, INFORMATION, OF, WITH, FROM, GIVEN,
        MODELS},
      keywordstyle=[1]\textcolor{OliveGreen},
    }
\begin{lstlisting}
%bql SIMULATE
...    MUTUAL INFORMATION OF
...    X WITH Y GIVEN W
...  FROM MODELS OF population
\end{lstlisting}
};


\node[
    rectangle,
    align = center,
    minimum width = 4cm,
    right = 1.cm of bql-query-header,
] (bayesdb-cmi-estimator-header) {\footnotesize
    \textbf{BayesDB Query Engine}
};

\node[
    rectangle,
    draw = none,
    fill = gray!40!white,
    align = center,
    minimum width = 4cm,
    minimum height = 1.6cm,
    below = 0cm of bayesdb-cmi-estimator-header,
] (bayesdb-cmi-estimator) {
    \textsc{CrossCat-Cmi}\\
    (Algorithm~\ref{alg:crosscat-cmi})
};

\draw[-stealth,thick]
    (bql-query) -- (bayesdb-cmi-estimator);

\draw[-stealth]
    (crosscat-structure-1.south) -- (bayesdb-cmi-estimator-header.north);

\draw[-stealth]
    (crosscat-structure-2.south) -- (bayesdb-cmi-estimator-header.north);

\draw[-stealth]
    (crosscat-structure-3.south) -- (bayesdb-cmi-estimator-header.north);


\node[
    rectangle,
    draw = none,
    right = .6cm of bayesdb-cmi-estimator.east,
] (ihat-structure-2) {\scriptsize
    $\I_{\hat\G_2}(X{:}Y|W)$
};

\node[
    rectangle,
    draw = none,
    above = .4cm of ihat-structure-2.south,
] (ihat-structure-1) {\scriptsize
    $\I_{\hat\G_1}(X{:}Y|W)$
};

\node[
    rectangle,
    draw = none,
    below = .4cm of ihat-structure-2.north,
] (ihat-structure-3) {\scriptsize
    $\I_{\hat\G_3}(X{:}Y|W)$
};

\node[
    rectangle,
    draw = none,
    align = left,
    below = .5cm of ihat-structure-3.north,
] (ihat-structure-dots) {\scriptsize
    $\dots$
};

\draw[-stealth] (bayesdb-cmi-estimator) -- (ihat-structure-2.west);
\draw[-stealth] (bayesdb-cmi-estimator) -- (ihat-structure-1.west);
\draw[-stealth] (bayesdb-cmi-estimator) -- (ihat-structure-3.west);


\node[
    rectangle,
    align = center,
    minimum width = 4cm,
    right = 2.75cm of bayesdb-cmi-estimator-header,
] (cmi-curve-header) {\footnotesize
    \textbf{CMI Posterior Distribution}
};

\node[
    rectangle,
    inner sep = 0pt,
    below = 0 cm of cmi-curve-header,
] (cmi-curve) {
    \includegraphics[width=5cm]{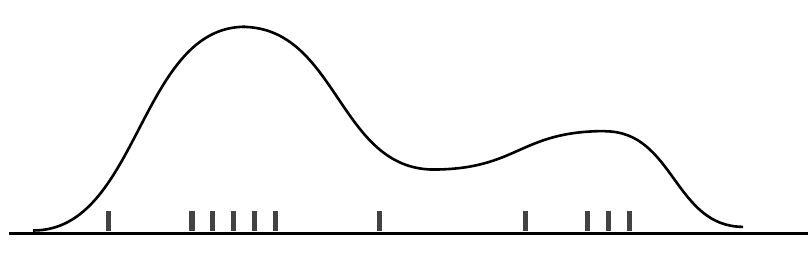}
};

\draw[-stealth] (ihat-structure-2.east) -- (cmi-curve);
\draw[-stealth] (ihat-structure-1.east) -- (cmi-curve);
\draw[-stealth] (ihat-structure-3.east) -- (cmi-curve);







\end{tikzpicture}

\caption{%
Workflow for computing posterior distributions of the CMI for variables in a
data table using BayesDB.
Modeling and inference in BayesDB produces an ensemble of posterior CrossCat
samples.
Each model learns a factorization of the joint distribution of all variables in
the database, and a Dirichlet process mixture within each block of
dependent variables.
For instance, model $\hat\G_1$ specifies that $X$ is independent of $(Y,W)$
which in turn is independent of $Z$, while in $\hat\G_3$, all variables are
(structurally) dependent.
End-user queries for the CMI are expressed in the Bayesian Query Language.
The BQL interpreter uses CrossCat structures to optimize the query where
possible, by (i) bypassing Monte Carlo estimation completely when the queried
variables are structurally independent, and/or (ii) dropping redundant
constraints which are structurally independent of the queried variables.
Values of CMI returned by each model constitute samples from the posterior
CMI distribution.}
\label{fig:workflow}

\end{figure*}

\section{Drawing Bayesian inferences about conditional mutual information}
\label{sec:bayesian-cmi}

Let $\x=(x_1,x_2,\dots,x_D)$ denote a $D$-dimensional random vector, whose
\mbox{sub-vectors} we denote $\x_\mtA = \set{x_i: i \in \mtA}$ with joint
probability density $\pt(\x_\mtA)$.
The symbol $\G$ refers to an arbitrary specification for the ``generative''
process of $\x$, and parameterizes all its joint and conditional densities.
The \textit{mutual information} (MI) of the variables $\x_\mtA$ and $\x_\mtB$
(under generative process $\G$) is defined in the usual way \cite{cover2012}:
\begin{align}
&\I_\G(\x_\mtA{:}\x_\mtB) =
  \mathop\mathbb{E}_{(\x_\mtA,\x_\mtB)}
    \left[ \log \left(
      \frac{\pt(\x_\mtA,\x_\mtB)}
      {\pt(\x_\mtA)\pt(\x_\mtB)}
    \right)
  \right].
\label{eq:def-mi}
\end{align}
The mutual information can be interpreted as the KL-divergence from
the product of marginals $\pt(\x_\mtA)\pt(\x_\mtB)$ to the joint
distribution $\pt(\x_\mtA,\x_\mtB)$, and is a well-established measure
for both the existence and strength of dependence between $\x_\mtA$
and $\x_\mtB$ (Section~\ref{subsec:extracting-cmi}).
Given an observation of the variables $\set{\x_\mtC{=}\hx_\mtC}$, the
\textit{conditional mutual information} (CMI) of $\x_\mtA$ and $\x_\mtB$ given
$\set{\x_\mtC{=}\hx_\mtC}$ is defined analogously:
\begin{multline}
\I_{\G}(\x_\mtA{:}\x_\mtB|\x_\mtC{=}\hx_\mtC) = \\
\mathop\mathbb{E}_{(\x_\mtA,\x_\mtB)|\hx_\mtC}
    \left[ \log \left(
      \frac{\pt(\x_\mtA,\x_\mtB|\hx_\mtC)}
      {\pt(\x_\mtA|\hx_\mtC)\pt(\x_\mtB|\hx_\mtC)}
    \right)
  \right].
\label{eq:def-cmi}
\end{multline}
Estimating the mutual information between the variables of $\x$ given a dataset
of observations $\D$ remains an open problem in the literature.
Various parametric and non-parametric methods for estimating MI
exist \cite{moddemeijer1989,moon1995,kraskov2004}; see \cite{paninski2003} for
a comprehensive review.
Traditional approaches typically construct a point estimate
$\hat\I(\x_\mtA{:}\x_\mtB)$ (and possible confidence intervals) assuming a
``true value'' of $\I(\x_\mtA{:}\x_\mtB)$.
In this paper, we instead take a non-parametric Bayesian approach, where the
mutual information itself is a derived random variable; a similar interpretation
was recently developed in independent work \cite{kunihama2016}.
The randomness of mutual information arises from treating the data generating
process and parameters $\G$ as a random variable, whose prior distribution we
denote $\pi$.
Composing $\G$ with the function $h: \hat\G
\mapsto \I_{\hat\G}(\x_\mtA{:}\x_\mtB)$ induces the derived random variable
$h(\G) \equiv \I_{\G}(\x_\mtA{:}\x_\mtB)$.
The distribution of the MI can thus be expressed as an expectation under
distribution $\pi$:
\begin{align}
\mathbb{P} \left[ \I_{\G}(\x_\mtA{:}\x_\mtB) \in S \right]
&= \int \mathbb{I} \left[ \I_{\hat\G} (\x_\mtA{:}\x_\mtB) \in S \right]
  \pi(d\hat\G) \notag\\
&= \mathop\mathbb{E}_{\hat\G\sim\pi} \left[
    \mathbb{I} \left[ \I_{\hat\G} (\x_\mtA{:}\x_\mtB) \in S \right]
  \right].
\label{eq:mi-posterior}
\end{align}
Given a dataset $\D$, we define the posterior distribution of the mutual
information, $\mathbb{P}\left[\I_{\G}(\x_\mtA{:}\x_\mtB) \in S |\D\right]$ as
the expectation in Eq~\eqref{eq:mi-posterior} under the posterior
$\pi(\cdot|\D)$.
We define the distribution over conditional mutual information
$\mathbb{P} \left[ \I_{\G}(\x_\mtA{:}\x_\mtB|\hx_\mtC) \in S \right]$
analogously to Eq~\eqref{eq:mi-posterior}, substituting the CMI
\eqref{eq:def-cmi} inside the expectation.

\subsection{Estimating CMI with generative population models}
\label{subsec:estimating-cmi}

Monte Carlo estimates of CMI can be formed for models expressed as
\textit{generative population models} \cite{mansinghka2015,saad2016}, a
probabilistic programming formalism for characterizing the data generating
process of an infinite array of realizations of random vector
$\x=(x_1,x_2,\dots,x_D)$.
Listing~\ref{lst:gpm-interface} summarizes elements of the GPM interface.
\begin{algorithm}[H]
\small
\floatname{algorithm}{Listing}
\caption{GPM interface for simulating from and assessing the density of
conditional and marginal distributions of a random vector $\x$.}
\label{lst:gpm-interface}
\textsc{Simulate}(%
  $\G$,
  query: $\mtQ{=}\set{q_j}$,
  condition: $\hx_\mtE{=}\set{\hat{x}_{e_j}}$)

\quad Return a sample $\mathbf{s} \sim \pt(\x_\mtQ|\hx_\mtE, \D).$

\medskip

\textsc{LogPdf}(%
  $\G$,
  query: $\hx_\mtQ{=}\set{\hat{x}_{q_j}}$,
  condition: $\hx_\mtE{=}\set{\hat{x}_{e_j}}$)

\quad Return the joint log density $p_\G(\hx_\mtQ|\hx_\mtE,\D)$
\end{algorithm}
\vspace{-.5cm}
These two interface procedures can be combined to derive a simple Monte Carlo
estimator for the CMI \eqref{eq:def-cmi}, shown in Algorithm~\ref{alg:gpm-cmi}.
\begin{subalgorithms}
\small
\captionof{algorithm}{\textsc{Gpm-Cmi}}
\label{alg:gpm-cmi}
\begin{algorithmic}[1]
\Require{%
  GPM $\G$; query $\mtA$, $\mtB$; condition $\hx_\mtC$; accuracy $T$}
\Ensure Monte Carlo estimate of
  $\I_\G\left(
    \x_\mtA{:}\x_\mtB|\x_\mtC{=}\hx_\mtC
  \right)$
\For{$t=1,\dots,T$}
  \State $(\hx_\mtA, \hx_\mtB)
    \gets \textsc{Simulate}(%
      \G, \mtA\cup\mtB, \hx_\mtC)$
  \State $m^t_{\mtA\cup\mtB} \gets \textsc{LogPdf}(%
      \G,\hx_{\mtA\cup\mtB}, \hx_\mtC)$
  \State $m^t_\mtA \gets \textsc{LogPdf}(%
      \hx_\mtA, \hx_\mtC)$
  \State $m^t_\mtB \gets \textsc{LogPdf}(%
      \hx_\mtB, \hx_\mtC)$
\EndFor
\State \Return $\frac{1}{T}\sum_{t=1}^T\left(
    m^t_{\mtA\cup\mtB} - (m^t_\mtA + m^t_\mtB)
  \right)$
\end{algorithmic}
\hrule
\end{subalgorithms}
While \textsc{\small Gpm-Cmi} is an unbiased and consistent estimator applicable
to any probabilistic model implemented as a GPM, its quality in detecting
dependencies is tied to the ability of $\G$ to capture patterns from the dataset
$\D$; this paper uses baseline non-parametric GPMs built using CrossCat
(Section~\ref{sec:gpms}).

\subsection{Extracting conditional independence relationships from CMI
estimates}
\label{subsec:extracting-cmi}

An estimator for the CMI can be used to discover several forms of independence
relations of interest.
\paragraph{Marginal Independence} It is straightforward to see that $(\x_\mtA
\indep_{\G} \x_\mtB)$ if and only if $\I_{\G}(\x_\mtA{:}\x_\mtB) = 0$.

\paragraph{Context-Specific Independence} If the event
$\set{\x_\mtC{=}\hx_\mtC}$ decouples $\x_\mtA$ and $\x_\mtB$, then
they are said to be independent ``in the context'' of $\hx_\mtC$, denoted
$(\x_\mtA \indep_\G \x_\mtB | \set{\x_\mtC{=}\hx_\mtC})$ \cite{boutilier1996}.
This condition is equivalent to the CMI from \eqref{eq:def-cmi} equaling zero.
Thus by estimating CMI, we are able to detect finer-grained independencies
than can be detected by analyzing the graph structure of a learned Bayesian
network \cite{shachter1998}.

\paragraph{Conditional Independence}
If context-specific independence holds for all possible observation sets
$\set{\x_\mtC{=}\hx_\mtC}$, then $\x_\mtA$ and $\x_\mtB$ are
\textit{conditionally independent} given $\x_\mtC$, denoted
$(\x_\mtA \indep_\G \x_\mtB | \x_\mtC)$.
By the non-negativity of CMI, conditional independence implies the
CMI of $\x_\mtA$ and $\x_\mtB$, marginalizing out $\x_\mtC$, is zero:
\begin{align}
\I_\G(\x_\mtA{:}\x_\mtB|\x_\mtC) =
  \mathop\mathbb{E}_{\hx_\mtC}\left[
    \I_\G(\x_\mtA{:}\x_\mtB|\x_\mtC=\hx_\mtC)
  \right] = 0.
\label{eq:def-cmi-marginal}
\end{align}
Figure~\ref{fig:vstruct-common} illustrates different CMI queries which are used
to discover these three types of dependencies in various data generators;
Figure~\ref{fig:bql} shows CMI queries expressed in the Bayesian
Query Language.


\section{Building generative population models for CMI estimation with
non-parametric Bayes}
\label{sec:gpms}


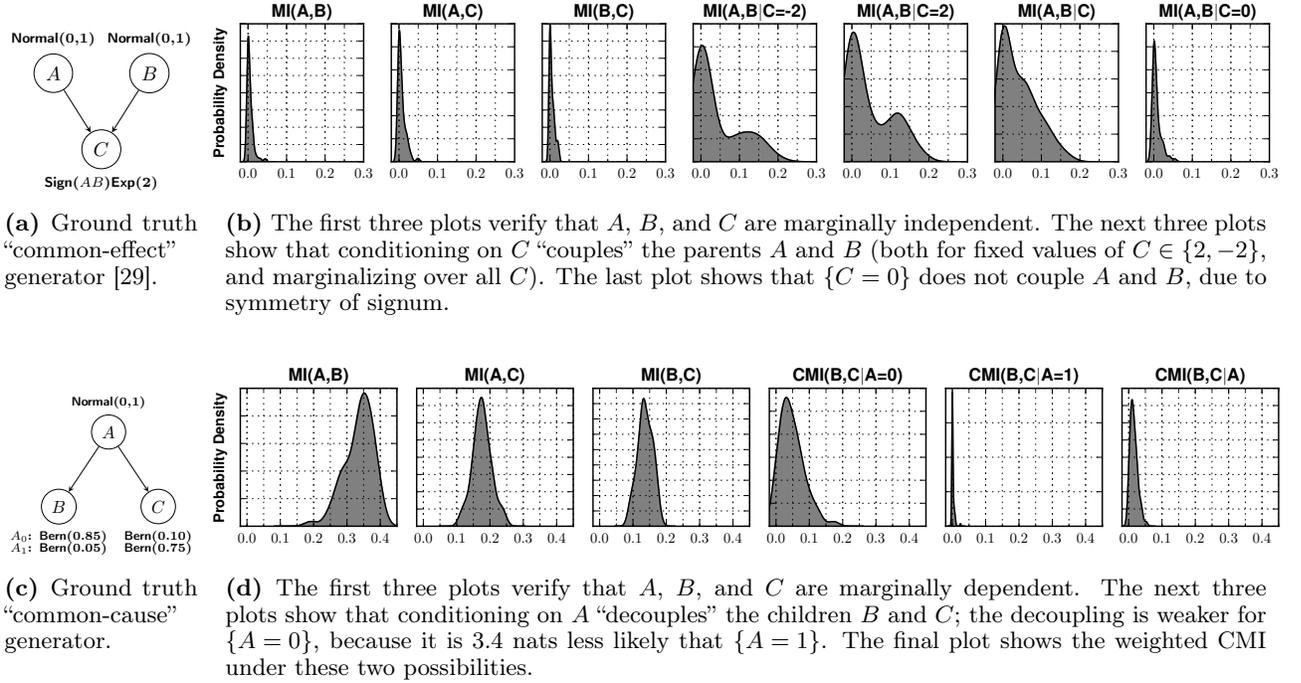
\begin{figure*}
\centering

\begin{subfigure}[t]{.15\linewidth}
\begin{tikzpicture}

\node[
    circle,
    draw = black,
] (A) {
    $A$
};

\node[
    circle,
    draw = black,
    right = 1 cm of A.east,
] (B) {
    $B$
};

\node[
    circle,
    draw = black,
    below = 1cm of A,
] (C) at ($(A)!0.5!(B)$) {
    $C$
};


\node[
    above = 0cm of A,
] (A-text) {\scriptsize\bf
    \textsf{Normal}(0,1)
};

\node[
    above = 0cm of B,
] (B-text) {\scriptsize\bf
    \textsf{Normal}(0,1)
};

\node[
    below = 0cm of C,
] (B-text) {\scriptsize\bf
    \textsf{Sign}($AB$)\textsf{Exp}(2)
};


\draw[-stealth] (A) -- (C);
\draw[-stealth] (B) -- (C);

\end{tikzpicture}
\subcaption{\small Ground truth ``common-effect'' generator
\cite{sejdinovic2013}.}
\label{fig:vstruct-graph}
\end{subfigure}%
\begin{subfigure}[t]{.85\linewidth}
\includegraphics[width=\linewidth]{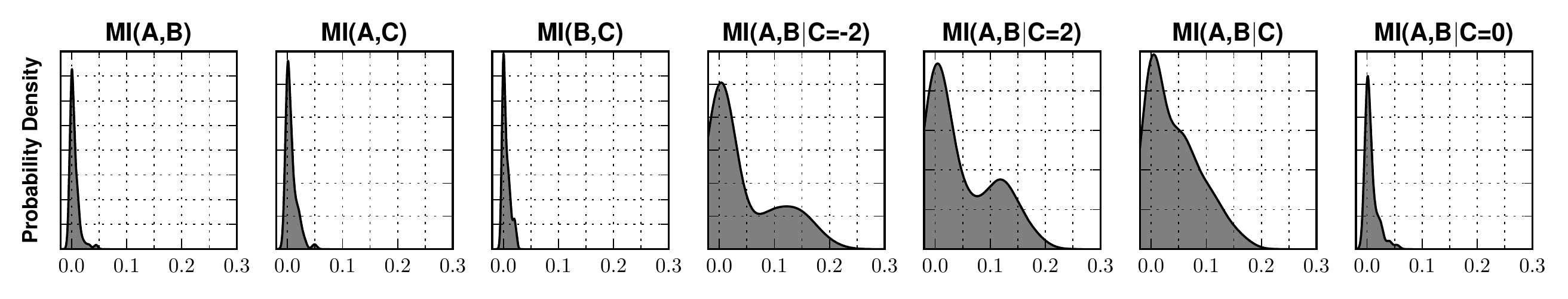}
\captionsetup{width=.95\linewidth}
\subcaption{The first three plots verify that $A$, $B$, and $C$ are marginally
independent. The next three plots show that conditioning on $C$ ``couples'' the
parents $A$ and $B$ (both for fixed values of $C\in\set{2,-2}$, and
marginalizing over all $C$). The last plot shows that $\set{C=0}$ does not
couple $A$ and $B$, due to symmetry of signum.}
\label{fig:vstruct-cmi}
\end{subfigure}

\bigskip

\begin{subfigure}[t]{.15\linewidth}
\begin{tikzpicture}

\node[
    circle,
    draw = black,
] (A) {
    $A$
};

\node[
    circle,
    draw = black,
    below left = 1cm and .5cm of A,
] (B) {
    $B$
};

\node[
    circle,
    draw = black,
    below right = 1cm and .5cm of A,
] (C) {
    $C$
};


\node[
    above = 0cm of A,
] (A-text) {\scriptsize\bf
    \textsf{Normal}(0,1)
};

\node[
    below = 0cm of B,
] (B-text-A0) {\scriptsize\bf
    $A_0$: \textsf{Bern}(0.85)
};

\node[
    below = 0.25cm of B,
] (B-text-A1) {\scriptsize\bf
    $A_1$: \textsf{Bern}(0.05)
};

\node[
    below = 0cm of C,
] (C-text-A0) {\scriptsize\bf
    \textsf{Bern}(0.10)
};

\node[
    below = 0.25cm of C,
] (C-text-A1) {\scriptsize\bf
    \textsf{Bern}(0.75)
};



\draw[-stealth] (A) -- (C);
\draw[-stealth] (A) -- (B);

\end{tikzpicture}
\subcaption{\small Ground truth ``common-cause'' generator.}
\label{fig:common-graph}
\end{subfigure}%
\begin{subfigure}[t]{.85\linewidth}
\includegraphics[width=\linewidth]{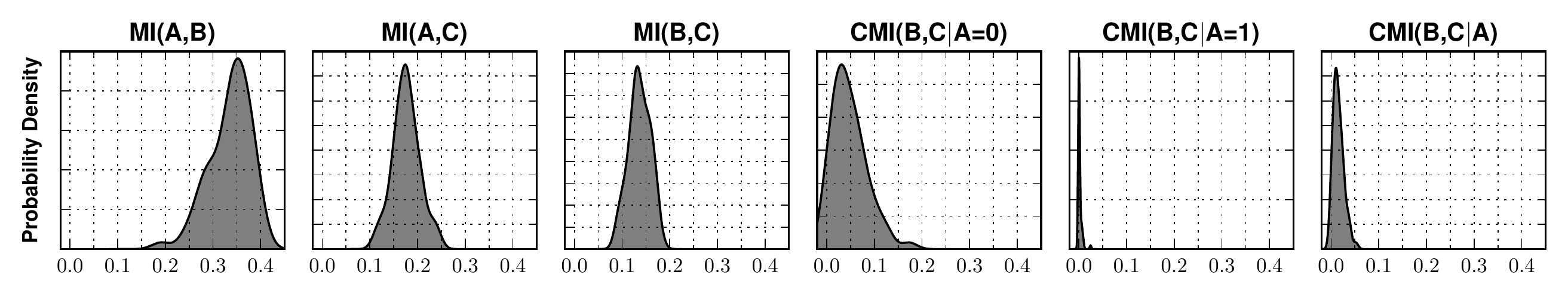}
\captionsetup{width=.95\linewidth}
\subcaption{The first three plots verify that $A$, $B$, and $C$ are marginally
dependent. The next three plots show that conditioning on $A$ ``decouples'' the
children $B$ and $C$; the decoupling is weaker for $\set{A=0}$, because it is
3.4 nats less likely that $\set{A=1}$. The final plot shows the weighted CMI
under these two possibilities.}
\label{fig:common-cmi}
\end{subfigure}

\caption{Posterior distributions of CMI under the DPMM posterior, given 100
data points from canonical Bayes net structures. Distributions peaked at 0
indicate high probability of (conditional) independence. In both cases, the
posterior CMI distributions correctly detect the marginal, conditional, and
context-specific independences in the ``ground truth'' Bayes nets, despite the
fact that both ``common-cause'' and ``common-effect'' structures are not in the
(structural) hypothesis space of the DPMM prior.}
\label{fig:vstruct-common}
\end{figure*}


\begin{figure*}[t]
\centering

\lstset{
  basicstyle=\ttfamily\footnotesize,
  columns=fullflexible,
  keepspaces=true,
  alsoletter={\.,\%},
  morekeywords=[1]{CREATE, TABLE, FROM, FOR, AND, EXPOSE, GUESS, POPULATION,
    METAMODEL, ANALYZE, INITIALIZE, INFER, EXPLICIT, PREDICT, USING,
    SAMPLES, MODEL, OVERRIDE, GENERATIVE, ITERATION, SIMULATE, SELECT,
    VARIABLES, CONDITIONAL, PROBABILITY, MUTUAL, INFORMATION, WITH, LIMIT,
    WHERE, OF, BY, MODELS, GIVEN, ESTIMATE},
  keywordstyle=[1]\textcolor{OliveGreen},
  morekeywords=[2]{NUMERICAL, NOMINAL},
  keywordstyle=[2]\textcolor{blue},
  morekeywords=[3]{probabilistic_pca},
  keywordstyle=[3]\textcolor{red},
  morekeywords=[4]{.scatter},
  keywordstyle=[4]\textcolor{cyan},
  morekeywords=[5]{\%mml, \%bql},
  keywordstyle=[5]\textcolor{magenta},
  morecomment=[l]{----},
  commentstyle=\color{gray},
  upquote=true,
  morestring=[b]',
  stringstyle=\color{orange}
}

\def\bordercolor{none}

\def\distanceVerticalHeader{.1cm}

\def\distanceVertical{0.1cm}
\def\distanceLabelPlot{0.3cm}
\def\distanceHorizontal{0.1cm}
\def\distanceHorizontalSmall{0.05cm}

\def\imagewidth{0.3\textwidth}

\def\boxheight{3cm}
\def\boxheightsmall{2cm}

\def\boxwidthBQL{4.50cm}
\def\boxwidthText{4.25cm}
\def\boxwidthAlg{7.25cm}

\usetikzlibrary{positioning}

\begin{tikzpicture}[thick]


\node[
    draw=\bordercolor,
    rectangle,
    minimum width=\boxwidthText,
    minimum height=\boxheightsmall,
    text width=\boxwidthText,
    align=left
] (english_1){\small
Simulate from the posterior distribution of the mutual information of
$(x_1, x_2)$ with $x_3$, given $x_4=14$.
};

\node[
    draw=\bordercolor,
    rectangle,
    below=\distanceVertical of english_1,
    minimum width=\boxwidthText,
    minimum height=\boxheight,
    text width=\boxwidthText,
    align=left
] (english_2){\small
Estimate the probability that the mutual information of
$(x_1, x_2)$ with $x_3$, given $x_4=14$ and marginalizing over $x_5$,
is less than 0.1 nats.
};

\node[
    draw=\bordercolor,
    rectangle,
    below=\distanceVertical of english_2,
    minimum width=\boxwidthText,
    minimum height=\boxheight,
    text width=\boxwidthText,
    align=left
] (english_3){\small
Synthesize a hypothetical dataset with 100 records, including
only those variables which are probably independent of $x_2$.
};

\node[
    draw=\bordercolor,
    rectangle,
    right=\distanceHorizontal of english_1,
    minimum width=\boxwidthBQL,
    minimum height=\boxheightsmall,
    align=left
]
    (bql_1) {
\begin{lstlisting}
SIMULATE
  MUTUAL INFORMATION OF
    (x1, x2) WITH (x3)
  GIVEN (x4 = 14)
FROM MODELS OF population
\end{lstlisting}
};

\node[
    draw=\bordercolor,
    rectangle,
    below=\distanceVertical of bql_1,
    minimum width=\boxwidthBQL,
    minimum height=\boxheight,
    align=left
] (bql_2) {
\begin{lstlisting}
ESTIMATE PROBABILITY OF
  MUTUAL INFORMATION OF
    (x1, x2) WITH (x3)
    GIVEN (x4 = 14, x5)
  < 0.1
BY population
\end{lstlisting}
};

\node[
    draw=\bordercolor,
    rectangle,
    below=\distanceVertical of bql_2,
    minimum width=\boxwidthBQL,
    minimum height=\boxheight
] (bql_3) {
\begin{lstlisting}
SIMULATE (
  SELECT * FROM
  VARIABLES OF population
  WHERE PROBABILITY OF
    MUTUAL INFORMATION
    WITH x2 < 0.1
  > 0.9)
FROM population
LIMIT 100;
\end{lstlisting}
};

\node[
    draw=\bordercolor,
    rectangle,
    right=\distanceHorizontal of bql_1,
    minimum height=\boxheightsmall,
    minimum width=\boxwidthAlg,
    text width=\boxwidthAlg
] (image_1){
\begin{varwidth}{\linewidth}
\begin{algorithmic}[1]
\algrenewcommand\algorithmicindent{.5em}%
\scriptsize
\For{$\G_k \in \mathcal{M}$}
    \State $\I_{\G_k} \gets \textsc{Gpm-Cmi}(
        \G_k, \set{x_1,x_2}, \set{x_3}, \set{(x_4,14)})$
\EndFor
\State \Return $(\I_{\G_1}, \dots, \I_{\G_{|\mathcal{M}|}})$
\end{algorithmic}
\end{varwidth}
};

\node[
    draw=\bordercolor,
    rectangle,
    below=\distanceVertical  of image_1,
    minimum height=\boxheight,
    minimum width=\boxwidthAlg,
    text width=\boxwidthAlg
] (image_2) {
\begin{varwidth}{\linewidth}
\begin{algorithmic}[1]
\algrenewcommand\algorithmicindent{.5em}%
\scriptsize
\For{$\G_k \in \mathcal{M}$}
  \For{$t=1,\dots,T$}
    \State $\hat{x}_5^t \gets \textsc{Simulate}(\G_k, x_5, \set{(x_4,14)})$
    \State $\I_{\G_k}^t \gets \textsc{Gpm-Cmi}($
    \Statex[3] $\G_k, \set{x_1,x_2}, \set{x_3},
        \set{(x_4,14), (x_5,\hat{x}_5^t)})$
  \EndFor
  \State $\I_{\G_k} \gets \frac{1}{T}\sum_{t} \left(\I_{\G_k}^t\right)$
\EndFor
\State \Return $\frac{1}{|\mathcal{M}|}\sum_{j}
  \left( \mathbb{I}\left[ \I_{\G_k} < 0.1 \right] \right)$
\end{algorithmic}
\end{varwidth}
};

\node[
    draw=\bordercolor,
    rectangle,
    below=\distanceVertical  of image_2,
    minimum height=\boxheight,
    text width=\boxwidthAlg
] (image_3) {
\begin{varwidth}{\linewidth}
\begin{algorithmic}[1]
\algrenewcommand\algorithmicindent{.5em}%
\scriptsize
\State $\mathcal{S} \gets \varnothing$
\For{$x_i \in (x_1,\dots,x_D)$}
  \For{$\G_k \in \mathcal{M}$}
    \State $\I_{\G_k} \gets \textsc{Gpm-Cmi}(\G_k, x_i, x_2, \varnothing)$
  \EndFor
  \State $p_i \gets
    \frac{1}{|\mathcal{M}|}
        \sum_k{\mathbb{I}\left[ \I_{\G_k} < 0.1 \right]}$
  \If{$p_i > 0.9$}
    \State $\mathcal{S} \gets \mathcal{S} \cup \set{x_i}$
  \EndIf
\EndFor
\For{$t=1\dots,100$}
    \State $s_t \gets
        \textsc{Simulate}(\mathcal{M}, \mathcal{S}, \varnothing)$
\EndFor
\State \Return $(s_1,\dots,s_{100})$
\end{algorithmic}
\end{varwidth}
};

\node[
    draw=\bordercolor,
    rectangle,
    above=\distanceVerticalHeader of english_1,
    text width=\boxwidthText
](header_1) {
\textbf{\footnotesize English Summary of\\CMI Query}
};

\node[
    draw=\bordercolor,
    rectangle,
    above=\distanceVerticalHeader of bql_1,
    text width=\boxwidthBQL
] (header_2) {
\textbf{\footnotesize CMI Query in Bayesian Query Language}
};

\node[
    draw=\bordercolor,
    rectangle,
    above=\distanceVerticalHeader of image_1,
    text width=\boxwidthAlg
] (header_3) {
\textbf{\footnotesize Inference Algorithm Invoked by\\Query Interpreter}
};


\draw (header_1.south west) -- (header_3.south east);

\draw (english_1.south west) -- (image_1.south east);
\draw (english_2.south west) -- (image_2.south east);


\end{tikzpicture}
\hrule
\medskip
\caption{%
End-user CMI queries in the Bayesian Query Language for three data
analysis tasks;
(top) evaluating the strength of predictive relationships;
(middle) specifying the amount of evidence required for a ``predictively''
significant relationship;
(bottom) synthesizing a hypothetical population, censoring probably sensitive
variables.}
\label{fig:bql}

\end{figure*}

Our approach to estimating the CMI requires a prior $\pi$ and model class $\G$
which is flexible enough to emulate an arbitrary joint distribution over $\x$,
and tractable enough to implement Algorithm~\ref{alg:gpm-cmi} for
its arbitrary \mbox{sub-vectors}.
We begin with a Dirichlet process mixture model (DPMM) \cite{gorur2010}.
Letting $L_d$ denote the likelihood for variable $d$, $V_d$ a prior over the
parameters of $L_d$, and $\blambda_d$ the hyperparameters of $V_d$, the
generative process for $N$ observations
$\D{=}\set{\x_{[i,1:D]}:1 \le i \le N}$ is:
\vspace{-0.1cm}
\begin{flalign*}
&\underline{\textsc{\small DPMM-Prior}}\\
& \alpha \sim \textsc{\small Gamma}(1,1)
  && \\
& \z = (z_1,\dots,z_N) \sim \textsc{\small CRP}(\cdot|\alpha)
  && \\
& \bphi_{[d,k]} \sim V_d(\cdot|\blambda_d)
  && d \in [D], k\in\textsc{\small Unique}(\z)\\
& x_{[i,d]} \sim L_d(\cdot|\bphi_{[d,z_i]})
  && i \in [N],d \in [D]
\end{flalign*}
We refer to \cite{escobar1995,jain2012} for algorithms for posterior inference,
and assume we have a posterior sample $\hat\G=(\alpha,\z_{[1:N]},\set{\bphi_d})$
of all parameters in the DPMM.
To compute the CMI of an arbitrary query pattern
$\I_{\hat\G}(\x_\mtA{:}\x_\mtB|\x_\mtC{=}\hx_\mtC)$ using
Algorithm~\ref{alg:gpm-cmi}, we need implementations of \textsc{\small Simulate}
and \textsc{\small LogPdf} for $\hat\G$.
These two procedures are summarized in Algorithms~\ref{alg:dpmix-simulate},
\ref{alg:dpmix-logpdf}.
\begin{subalgorithms}
\small
\captionof{algorithm}{\textsc{\small DPMM-Simulate}}
\label{alg:dpmix-simulate}
\begin{algorithmic}[1]
\Require DPMM $\G$; target $\mtA$; condition $\hx_\mtC$
\Ensure joint sample $\hx_\mtA \sim \pt(\cdot|\hx_\mtC)$
\State $(l_i)_{i=1}^{K+1} \gets
  \textsc{DPMM-Cluster-Posterior}(\G,\hx_\mtC)$
\State $z_{N+1} \sim \textsc{Categorical}(l_1,\dots,l_{K+1})$
\For{$a \in \mtA$}
  \State $\hat{x}_a \sim L_a(\cdot|\bphi_{[a,z_{N+1}]})$
  \label{algline:dpmix-simulate-likelihood}
\EndFor
\State \Return $\hx_\mtA$
\end{algorithmic}
\vspace{-.3cm}

\captionof{algorithm}{\textsc{\small DPMM-LogPdf}}
\label{alg:dpmix-logpdf}
\begin{algorithmic}[1]
\Require DPMM $\G$; target $\hx_\mtA$; condition $\hx_\mtC$
\Ensure log density $\pt(\hx_\mtA|\hx_\mtC)$
\State $(l_i)_{i=1}^{K+1} \gets
  \textsc{DPMM-Cluster-Posterior}(\G, \hx_\mtC)$
\For{$k = 1,\dots,K+1$}
  \State $t_k \gets
    \prod_{a\in\mtA}L_a(\hat{x}_a|\bphi_{[a,k]})$
    \label{algline:dpmix-logpdf-likelihood}
\EndFor
\State \Return
  $\log\left(\sum_{k=1}^{K+1}\left(t_{k}l_{k}\right)\right)$
\end{algorithmic}
\vspace{-.3cm}
\captionof{algorithm}{\textsc{\small DPMM-Cluster-Posterior}}
\label{alg:dpmix-cluster-posterior}
\begin{algorithmic}[1]
\Require DPMM $\G$; condition $\hx_\mtC$;
\Ensure $\set{\pt(z_{N+1}=k):1\le{k}\le\max(\z_{1:N})+1}$
\State $K \gets \max(\z_{1:N})$
\For{$k=1,\dots,K+1$}
  \State $n_k \gets \begin{cases}
    |\set{\x_i\in\D: z_i=k}| & \text{if } k\le K\\
    \alpha & \text{if } k = K+1
    \end{cases}$
  \State $l_k \gets
    \left(
      \prod_{c\in\mtC}L_c(\hat{x}_c|\bphi_{[c,k]})
    \right)n_k$
\EndFor
\State \Return $(l_1,\dots,l_{K+1}) / \sum_{k=1}^{K+1}(l_k)$
\end{algorithmic}
\hrule
\end{subalgorithms}

The subroutine \textsc{\small DPMM-Cluster-Posterior} is used for
sampling (in \textsc{\small DPMM-Simulate}) and marginalizing over (in
\textsc{\small DPMM-LogPdf}) the non-parametric mixture components.
Moreover, if $L_d$ and $V_d$ form a conjugate likelihood-prior pair, then
invocations of $L_d(\hat{x}_d|\bphi_{[d,k]})$ in
Algorithms \ref{alg:dpmix-simulate}:\ref{algline:dpmix-simulate-likelihood} and
\ref{alg:dpmix-logpdf}:\ref{algline:dpmix-logpdf-likelihood} can be
\mbox{Rao-Blackwellized} by conditioning on the sufficient statistics of data
in cluster $k$, thus marginalizing out $\bphi_{[d,k]}$ \cite{robert2005}.
This optimization is important in practice, since analytical marginalization can
be obtained in closed-form for several likelihoods in the exponential family
\cite{fink1997}.
Finally, to approximate the posterior distribution over CMI in
\eqref{eq:def-cmi}, it suffices to aggregate
$\textsc{\small DPMM-Cmi}$ from a set of posterior samples
$\set{\hat\G_1,\dots,\hat\G_H} \sim^\textrm{iid} \pi(\cdot|\D)$.
Figure~\ref{fig:vstruct-common} shows posterior CMI distributions from the DPMM
successfully recovering the marginal and conditional independencies in two
canonical Bayesian networks.

\subsection{Inducing sparse dependencies using the CrossCat prior}

The multivariate DPMM makes the restrictive assumption that all variables
$\x=(x_1,\dots,x_D)$ are (structurally) marginally dependent, where their joint
distribution fully factorizes conditioned on the mixture assignment $z$.
In high-dimensional datasets, imposing full structural dependence among all
variables is too conservative.
Moreover, while the Monte Carlo error of
\mbox{Algorithm~\ref{alg:gpm-cmi}} does not scale with the dimensionality $D$,
its runtime scales linearly for the DPMM, and so estimating the CMI is
likely to be prohibitively expensive.
We relax these constraints by using CrossCat \cite{mansinghka2016}, a structure
learning prior which induces sparsity over the dependencies between the
variables of $\x$.
In particular, CrossCat posits a factorization of $\x$ according to a
\textit{variable partition} $\bgamma = \set{\mtV_1,\dots,\mtV_{|\bgamma|}}$,
where $\mtV_i \subseteq [D]$.
For $i\ne j$, all variables in block $\mtV_i$ are mutually (marginally and
conditionally) independent of all variables in $\mtV_j$.
The factorization of $\x$ given the variable partition $\bgamma$ is therefore
given by:
\begin{align}
\pt(\x |\D) &= \prod_{\mtV \in \bgamma} \ptv(\x_\mtV|\D_\mtV).
\label{eq:cc-factorization}
\end{align}
Within block $\mtV$, the variables $\x_\mtV = \set{x_d: d\in \mtV}$ are
distributed according to a multivariate DPMM; subscripts with $\mtV$ (such as
$\G_\mtV$) now index a set of \mbox{block-specific} DPMM parameters.
The joint predictive density $\ptv$ is given by %
Algorithm~\ref{alg:dpmix-logpdf}:
\vspace{-.1cm}
\begin{align}
\ptv(\x_\mtV | \D) =
  \sum_{k=1}^{K_V+1} \left(
    \frac{n_{[\mtV,k]}\prod\limits_{d\in \mtV} \ptv(x_d|\bphi_{[d,k]})}
    {\sum_{k'}n_{[\mtV,k']}}
  \right). \label{eq:cc-dpmm}
\end{align}
The CrossCat generative process for $N$ observations
$\D{=}\set{\x_{[i,1:D]}:1\le i \le N}$ is summarized below.
\begin{align*}
&\underline{\textsc{\small CrossCat-Prior}}\\
& \alpha' \sim \textsc{\small Gamma}(1,1)
  && \\
& \bv = (v_1,\dots,v_D) \sim \textsc{\small CRP}(\cdot|\alpha')
  && \\
& \mtV_k \gets \set{i \in [D]: v_i = k}
  && k\in\textsc{\small Unique}(\bv)\\
& \set{x_{[i,\mtV_k]}}_{i=1}^{N} \sim \textsc{\small DPMM-Prior}
  && k\in\textsc{\small Unique}(\bv)
\end{align*}
We refer to \cite{mansinghka2016,obermeyer2014} for algorithms for posterior
inference in CrossCat, and assume we have a set of approximate samples
$\set{\hat\G_i: 1\le i \le H}$ of all latent CrossCat parameters from the posterior
$\pi(\cdot|\D)$.

\subsection{Optimizing a CMI query}
The following lemma shows how CrossCat induces sparsity for a multivariate CMI
query.
\begin{lemma}
Let $\G$ be a posterior sample from CrossCat, whose full joint distribution is
given by \eqref{eq:cc-factorization} and \eqref{eq:cc-dpmm}. Then, for all
$\mtA,\mtB, \mtC \subseteq [D]$,
\begin{align*}
\I_\G \left( \x_\mtA : \x_\mtB \mid \hx_\mtC \right) =
  \sum_{\mtV\in\bgamma}
    \I_{\G_\mtV}(
      \x_{\mtA\cap\mtV} : \x_{\mtB\cap\mtV} \mid \hx_{\mtC\cap\mtV}
    ),
\end{align*}
where $\I_{\G_\mtV}
  (\x_{\mtA\cap{\mtV}} : \varnothing | \hx_{\mtC\cap\mtV} )\equiv 0$.
\label{lemma:crosscat-cmi-sparse}
\end{lemma}
\vspace{-.3cm}
\begin{proof}
\renewcommand{\qedsymbol}{}
Refer to Appendix~\ref{appx:proof-optimizing-cmi}.
\end{proof}


\begin{figure}[ht]
\centering
\includegraphics[width=\linewidth]{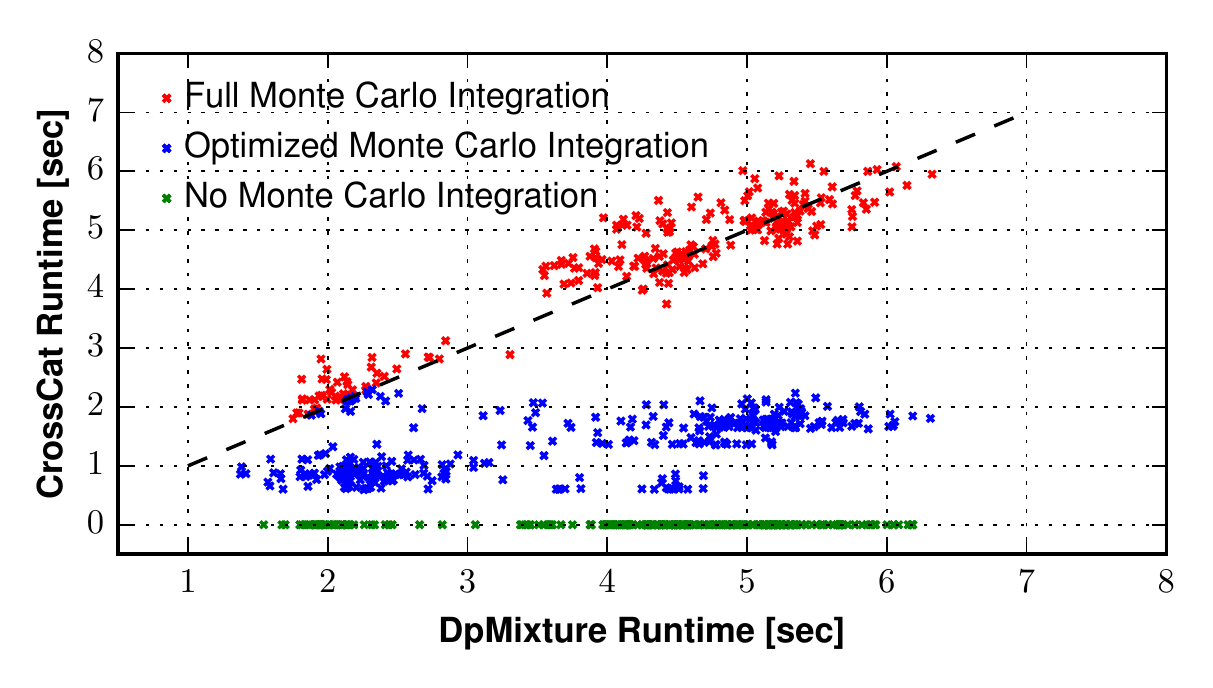}
\captionsetup{skip=-5pt}
\caption{\footnotesize Comparing the runtime of \textsc{\small CrossCat-Cmi}
\mbox{(Alg~\ref{alg:crosscat-cmi})} and \textsc{\small Gpm-Cmi}
\mbox{(Alg~\ref{alg:gpm-cmi})} (using the DPMM), on 1000 randomly generated
CMI queries from an 8-dimensional dataset.
The dashed curve shows the 45-degree line.
The green dots at 0 correspond to CrossCat detecting structural independence
between query variables, bypassing Monte Carlo estimation completely.
The blue dots (below diagonal) correspond to CrossCat optimizing the Monte
Carlo estimator by ignoring constraint variables which are structurally
independent of the target variables.
The red dots (along diagonal) correspond to CrossCat learning no structural
independences, requiring full Monte Carlo estimation and resulting in comparable
runtime to DPMM.
These three cases correspond to the three posterior CrossCat structures
illustrated in Figure~\ref{fig:workflow}, when the targets variables are
$X$ and $Y$ conditioned on $W$.}
\label{fig:runtime}
\hrule
\end{figure}

\vspace{-.3cm}
An immediate consequence of Lemma~\ref{lemma:crosscat-cmi-sparse} is that
structure discovery in CrossCat allows us to optimize Monte Carlo estimation of
$\I_\G(\x_\mtA:\x_\mtB|\x_\mtC=\hx_\mtC)$ by ignoring all target and
condition variables which are not in the same block $\mtV$, as shown in
Algorithm~\ref{alg:crosscat-cmi} and Figure~\ref{fig:runtime}.
\vspace{-.2cm}
\begin{subalgorithms}
\small
\captionof{algorithm}{\textsc{\small CrossCat-Cmi}}
\label{alg:crosscat-cmi}
\begin{algorithmic}[1]
\algrenewcommand\algorithmicindent{.5em}%
\Require{%
  CrossCat $\G$; query $\mtA$, $\mtB$; condition $\hx_\mtC$; acc. $T$}
\Ensure Monte Carlo estimate of $\I_\G(\x_\mtA{:}\x_\mtB|\x_\mtC{=}\hx_\mtC)$
\For{$\mtV \in \bgamma$}
  \If{ $\mtA\cap\mtV \;\; \textsc{\small and } \;\; \mtB\cap\mtV}$
    \State $i_\mtV \gets $\textsc{Gpm-Cmi}(%
      $\G_\mtV$,
      $\mtA \cap \mtV$,
      $\mtB \cap \mtV$,
      $\hx_{\mtC{\cap}{\mtV}}$,
      $T$)
  \Else
    \State $i_\mtV \gets 0$
  \EndIf
\EndFor
\State \Return $\sum_{\mtV\in\bgamma}i_\mtV$
\end{algorithmic}
\end{subalgorithms}

\subsection{Upper bounding the pairwise dependence probability}
\label{sub:crosscat-fast}

In exploratory data analysis, we are often interested in detecting pairwise
predictive relationships between variables $(x_i, x_j)$.
Using the formalism from Eq \eqref{eq:mi-posterior}, we can compute the
probability that their MI is non-zero: $\mathbb{P}\left[\I_{\G}(x_i:x_j) >
0\right]$.
This quantity can be upper-bounded by the posterior probability that $x_i$
and $x_j$ have the same assignments $v_i$ and $v_j$ in the CrossCat variable
partition $\bgamma$:
\begin{align}
&\mathbb{P} \left[ \I_{\G}(x_i : x_j) > 0 \right] && \notag \\
&=\mathbb{P} \left[ \I_{\G}(x_i : x_j) > 0\mid\set{\G:v_i=v_j} \right]
  \mathbb{P}\left[ \set{\G:v_i=v_j} \right] \notag &&\\
&\;\; +
    \mathbb{P}\left[ \I_{\G}(x_i:x_j) > 0\mid \set{\G:v_i\ne v_j} \right]
    \mathbb{P}\left[ \set{\G:v_i\ne v_j} \right] \notag &&\\
&= \mathbb{P}\left[ \I_{\G}(x_i : x_j) > 0\mid \set{\G:v_i=v_j} \right]
  \mathbb{P}\left[ \set{\G:v_i=v_j} \right] \notag &&\\
&<\mathbb{P} \left[ \set{\G:v_i=v_j} \right]
  \approx \frac{1}{H} \sum_{h=1}^H
    \mathbb{I}[\hat\G_h: \hat{v}_{[h,i]} = \hat{v}_{[h,j]}], \label{eq:depprob}
\end{align}
where Lemma~\ref{lemma:crosscat-cmi-sparse} has been used to set the addend in
line 3 to zero.
Also note that the summand in \eqref{eq:depprob} can be computed in $O(1)$ for
CrossCat sample $\hat\G_h$.
When dependencies among the $D$ variables are sparse such that many pairs
$(x_i,x_j)$ have MI upper bounded by 0, the number of invocations of
Algorithm~\ref{alg:crosscat-cmi} required to compute pairwise MI values is $\ll
O(D^2)$.
A comparison of upper bounding MI versus exact MI estimation with Monte Carlo is
shown in Figure~\ref{fig:depprob}.


\begin{figure}[ht]
\centering
\includegraphics[width=\linewidth]{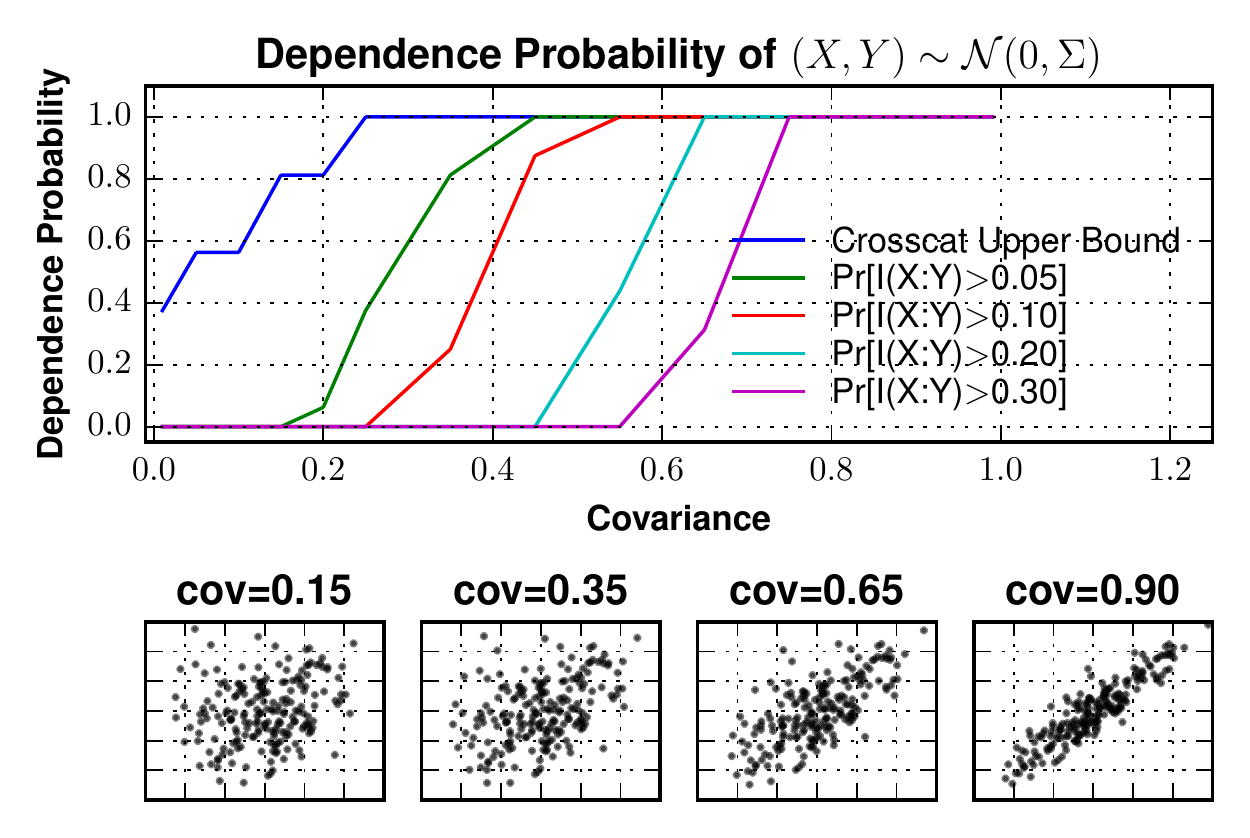}
\captionsetup{skip=0pt}
\caption{\footnotesize  Posterior probability that dimensions of a bivariate
Gaussian are dependent, vs the covariance (top). The CrossCat upper bound
\eqref{eq:depprob} is useful for detecting the existence of a predictive
relationship; the posterior distribution of MI can determine whether the
strength of the relationship is ``predictively significant'' based on various
tolerance levels (0.05, 0.10, 0.20, and 0.30 nats).}
\label{fig:depprob}
\hrule
\end{figure}


\section{Applications to macroeconomic indicators of global poverty,
education, and health}
\label{sec:applications}


\begin{figure*}[t]

\begin{subfigure}[b]{\linewidth}
    \begin{subfigure}[b]{.33\linewidth}
    \centering
    \includegraphics[width=\linewidth]{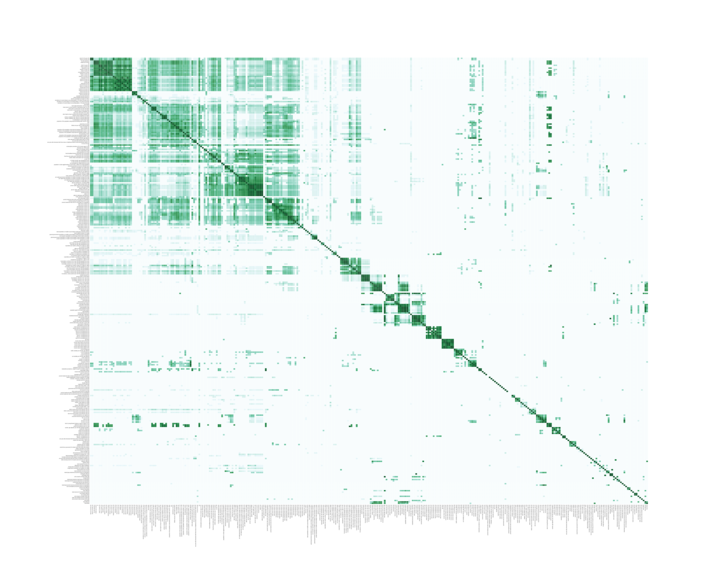}
    \captionsetup{skip=-5pt}
    \subcaption{$R^2$ Correlation Value}
    \label{subfig:heatmap-pearson-bonferroni}
    \end{subfigure}%
    \begin{subfigure}[b]{.33\linewidth}
    \includegraphics[width=\linewidth]{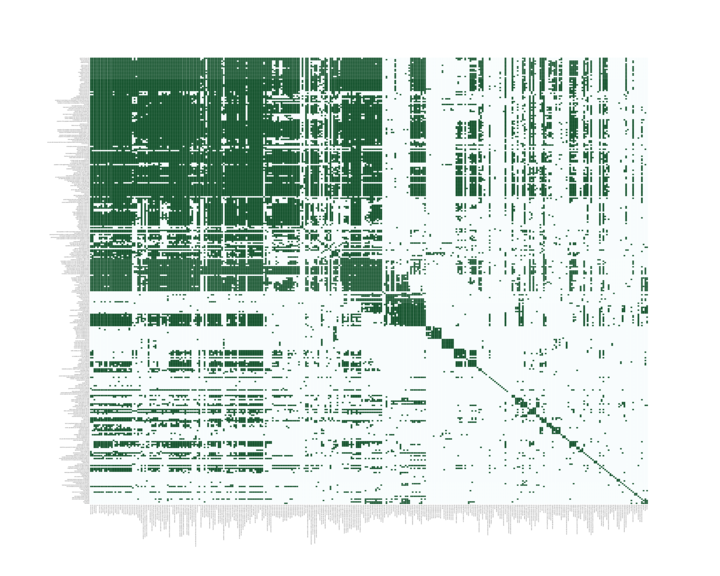}
    \captionsetup{skip=-5pt}
    \subcaption{HSIC \cite{gretton2005} Independence Test}
    \label{subfig:heatmap-hsic}
    \end{subfigure}%
    \begin{subfigure}[b]{.33\linewidth}
    \includegraphics[width=\linewidth]{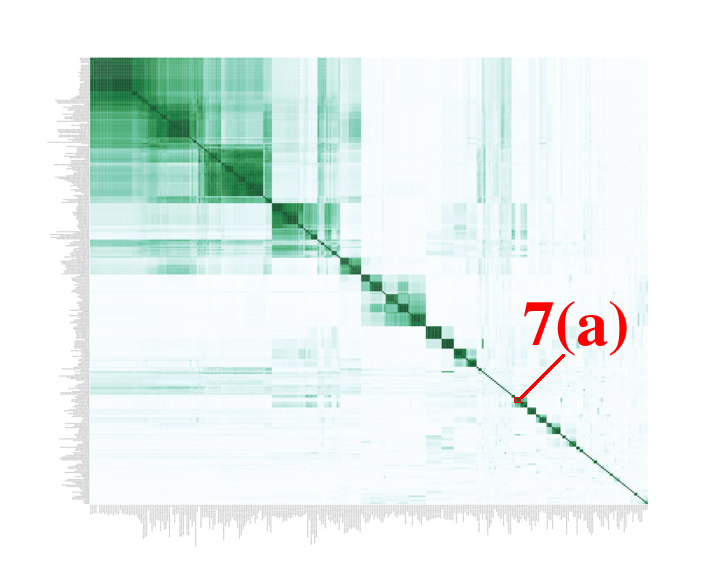}
    \captionsetup{skip=-5pt}
    \subcaption{$\mathbb{P}[\I_\G(x_i{:}x_j)>0]$, Eq~\eqref{eq:depprob}}
    \label{subfig:heatmap-dependence-probability}
    \end{subfigure}
\caption{Pairwise heatmaps of all 320 variables in the Gapminder dataset
using three dependency detection techniques. Darker cells indicate
a detected dependence between the two variables.}
\end{subfigure}%
\vspace{-0.1in}

\begin{subfigure}{\linewidth}
\begin{subtable}{.6\linewidth}
\tiny\bf
\begin{tabular*}{\linewidth}{lll@{\extracolsep{\fill}}l}
\toprule
variable A & variable B & $\mathbb{P}[\I_\G>0]$ & $R^2$ ($p\ll10^{-6}$)  \\ \midrule
personal computers & earthquake affected & 0.015625 &  0.974445 \\
road traffic total deaths & people living w/ hiv & 0.265625 & 0.858260 \\
natural gas reserves & 15-25 yrs sex ratio & 0.031250 & 0.951471 \\
forest products per ha & earthquake killed & 0.046875 & 0.936342 \\
flood affected & population 40-59 yrs & 0.140625 & 0.882729 \\ \bottomrule
\end{tabular*}
\end{subtable}%
\begin{subfigure}{.4\linewidth}
\includegraphics[width=\linewidth]{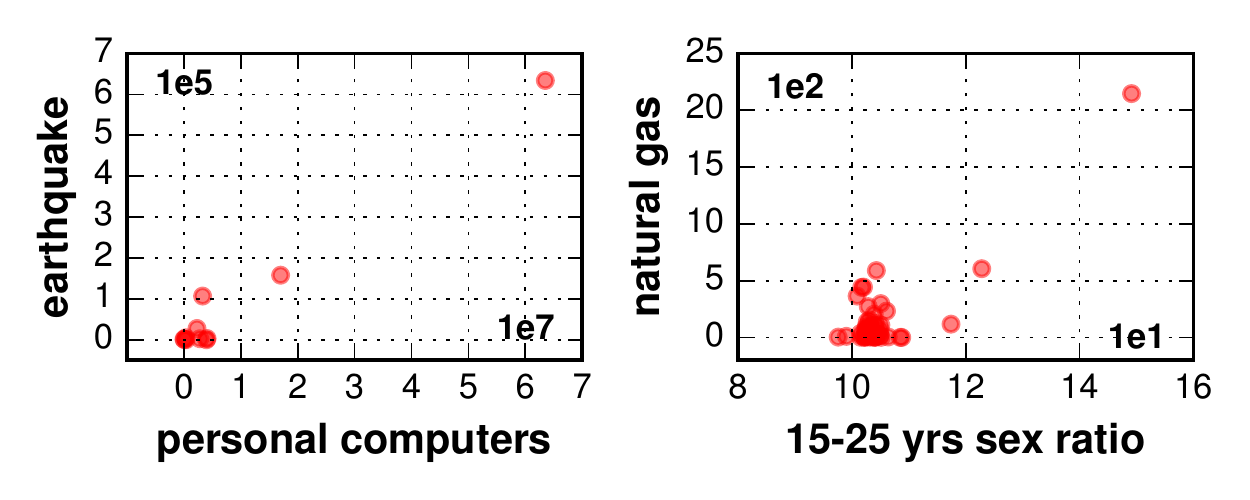}
\end{subfigure}
\captionsetup{skip=0pt}
\caption{Spurious relationships which are correlated under $R^2$, but
probably independent according to posterior MI.}
\label{tab:r2-cmi-spurious}
\end{subfigure}

\begin{subfigure}{\linewidth}
\begin{subtable}{.6\linewidth}
\tiny\bf
\begin{tabular*}{\linewidth}{lll@{\extracolsep{\fill}}l}
\toprule
variable A & variable B & $\mathbb{P}[\I_\G>0]$ & $R^2$ ($p\ll10^{-6}$) \\ \midrule
inflation & trade balance (\% gdp) & 0.859375 & 0.114470 \\
DOTS detection rate & DOTS coverage & 0.812500 & 0.218241 \\
forest area (sq. km) &  forest products (usd) &  0.828125 & 0.206145 \\
long term unemp. rate & total 15-24 unemp. & 0.968750 & NaN \\
male family workers & female self-employed & 0.921875 & NaN \\\bottomrule
\end{tabular*}
\end{subtable}%
\begin{subfigure}{.4\linewidth}
\includegraphics[width=\linewidth]{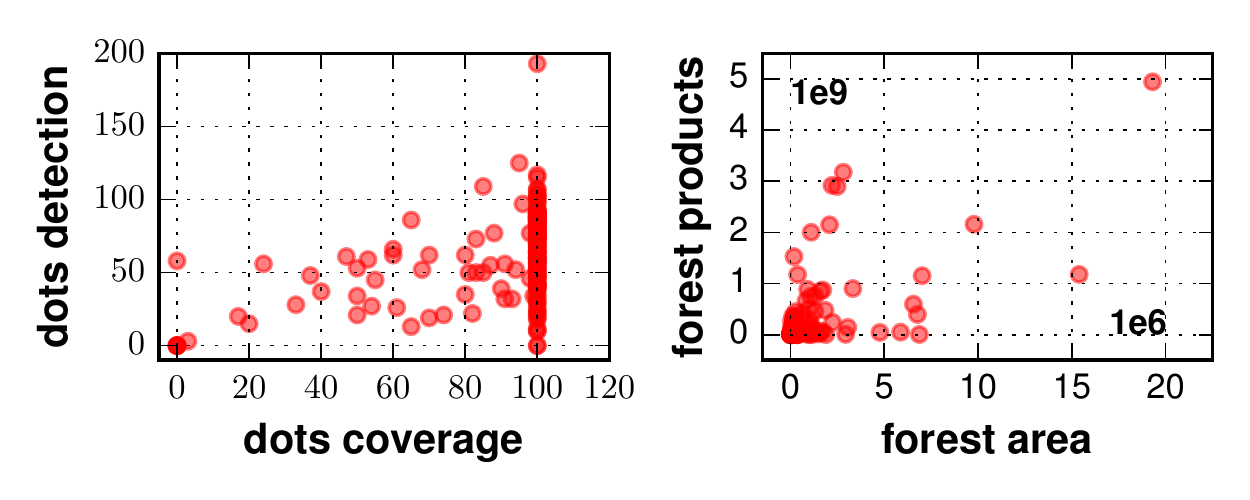}
\end{subfigure}
\captionsetup{skip=0pt}
\caption{Common-sense relationships probably dependent according to
posterior MI, but weakly correlated under $R^2$.}
\label{tab:r2-cmi-common-sense}
\end{subfigure}

\caption{Comparing dependences between variables in the Gapminder dataset,
as detected by $R^2$, HSIC (with Bonferroni correction for multiple testing),
and posterior distribution over mutual information in CrossCat.}
\label{fig:gapminder}
\end{figure*}

\begin{figure*}[t]
\centering

\begin{subfigure}[b]{.3\linewidth}
\includegraphics[width=\linewidth]{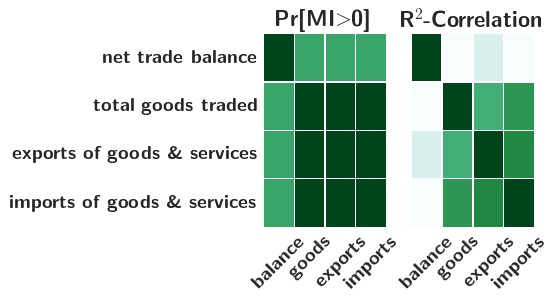}
\captionsetup{skip=2pt, width=.95\linewidth}
\subcaption{Block of ``trade'' variables detected as probably dependent.}
\label{fig:trade-heatmap}
\end{subfigure}%
\begin{subfigure}[b]{.4\linewidth}
\includegraphics[width=\linewidth]{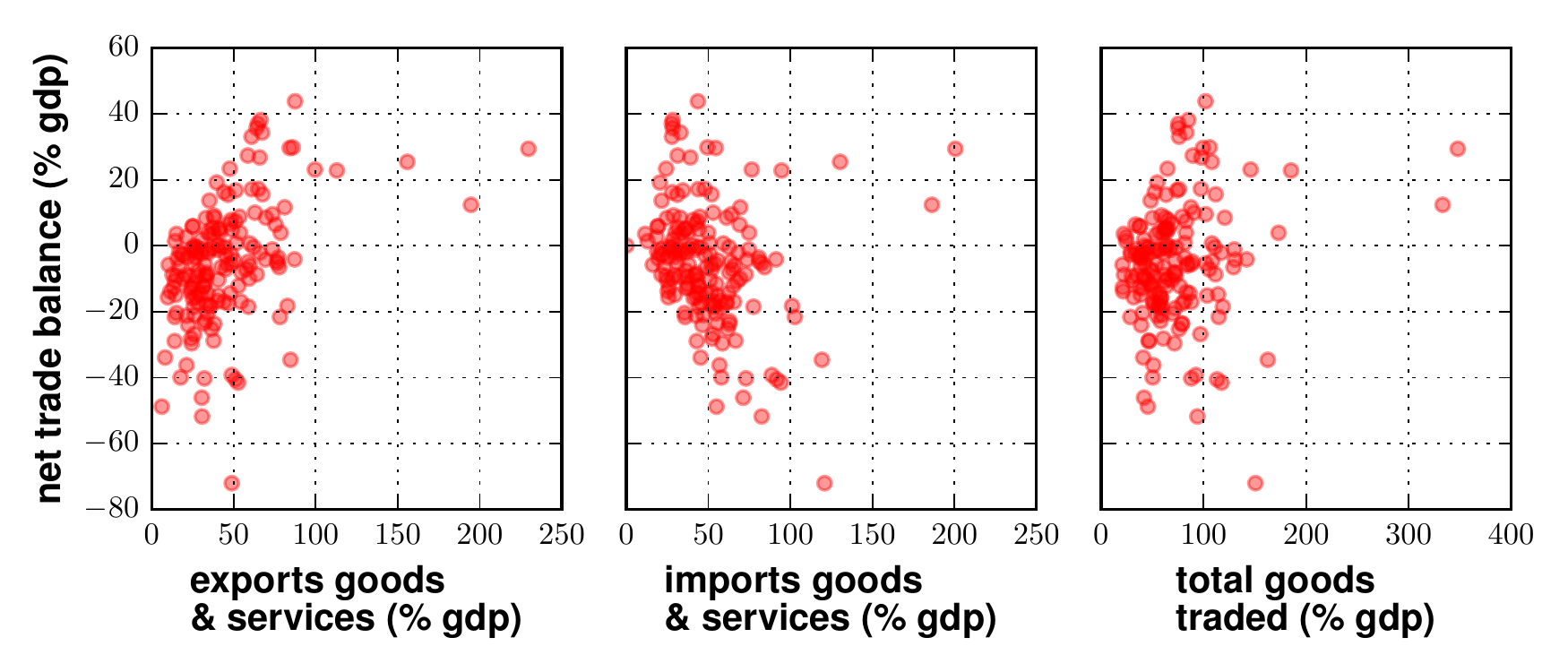}
\captionsetup{skip=2pt, width=.95\linewidth}
\subcaption{Scatter plots show weak linear correlations and
heteroskedastic noise for \texttt{net balance}.}
\label{fig:trade-scatter}
\end{subfigure}%
\begin{subfigure}[b]{.3\linewidth}
\includegraphics[width=\linewidth]{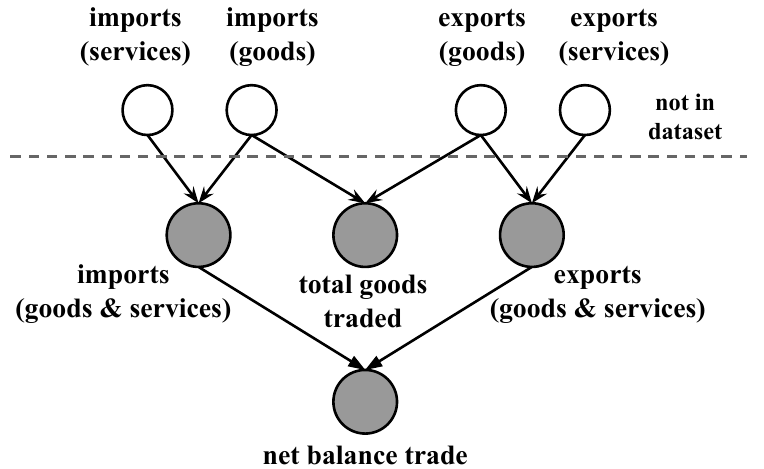}
\captionsetup{skip=2pt, width=.95\linewidth}
\subcaption{Ground truth causal structure of variables in the ``trade'' block.}
\label{fig:trade-graph}
\end{subfigure}

\begin{subfigure}[b]{\linewidth}
\centering
\includegraphics[width=\linewidth]{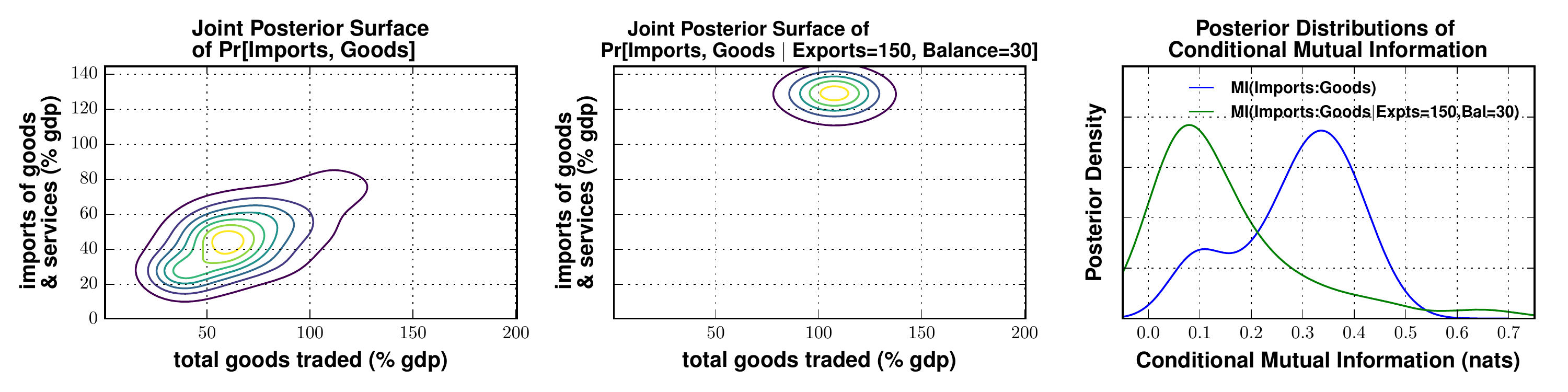}
\captionsetup{skip=0pt}
\subcaption{Left plot shows the joint posterior density of ``imports'' and
``goods'', where the marginal coupling is due to their common parent in
\subref{fig:trade-graph}. Center plot shows the same distribution conditioned on
``exports''{=}150 and ``balance''=30; ``imports'' now centers around its
noiseless value of 120, and is decoupled from ``goods''. Right plot shows the
CMI for these distributions.}
\label{fig:trade-joint-cmi}
\end{subfigure}

\caption{CMI discovers existence and confirms strength of predictive
relationships between ``trade'' variables.}
\label{fig:gapminder-trade}

\end{figure*}

This section illustrates the efficacy of the proposed approach on a sparse
database from an ongoing collaboration with the Bill \& Melinda Gates
Foundation.%
\footnote{A further application, to a \mbox{real-world} dataset of mathematics
exam scores, is shown in Appendix~\ref{appx:applications-marks}.}
The Gapminder data set is an extensive longitudinal dataset of $\sim$320 global
developmental indicators for 200 countries spanning over 5 centuries
\cite{rosling2008}.
These include variables from a broad set of categories such as education,
health, trade, poverty, population growth, and mortality rates.
We experiment with a cross-sectional slice of the data from 2002.
Figure~\ref{subfig:heatmap-pearson-bonferroni} shows the pairwise $R^2$
correlation values between all variables; each row and column in the heatmap is
an indicator in the dataset, and the color of a cell is the raw value of $R^2$
(between 0 and 1).
Figure~\ref{subfig:heatmap-hsic} shows pairwise binary hypothesis tests of
independence using HSIC \cite{gretton2005}, which detects a dense set of
dependencies including many spurious relationships %
(Appendix~\ref{appx:experimental-methods}).
For both methods, statistically insignificant relationships ($\alpha=0.05$ with
Bonferroni correction for multiple testing) are shown as 0.
Figure~\ref{subfig:heatmap-dependence-probability} shows an upper bound on the
pairwise probability that the MI of two variables exceeds zero (also a value
between 0 and 1).
These entries are estimated using Eq~\eqref{eq:depprob} (bypassing Monte Carlo
estimation) using $H{=}100$ samples of CrossCat.
Note that the metric $\mathbb{P}[\I_\G(x_i{:}x_j)>0]$ in
\mbox{Figure~\ref{subfig:heatmap-dependence-probability}} only indicates the
\textit{existence} of a predictive relationship between $x_i$ and $x_j$; it
does not quantify either the strength or directionality of the relationship.

It is instructive to compare the dependencies detected by $R^2$ and
\textsc{\small CrossCat-Cmi}.
Table \ref{tab:r2-cmi-spurious} shows pairs of variables that are spuriously
reported as dependent according to correlation; scatter plots reveal they are
either (i) are sparsely observed or (ii) exhibit high correlation due to large
outliers.
Table \ref{tab:r2-cmi-common-sense} shows common-sense relationships between
pairs of variables that \textsc{\small{CrossCat-CMI}} detects but $R^2$ does
not; scatter plots reveal they are either (i) non-linearly related, (ii)
roughly linear with heteroskedastic noise, or (iii) pairwise independent but
dependent given a third variable.
Recall that CrossCat is a product of DPMMs; practically meaningful
conditions for weak and strong consistency of Dirichlet location-scale mixtures
have been established by \cite{ghosal1999,tokdar2006}.
This supports the intuition that CrossCat can detect a broad class of predictive
relationships that simpler parametric models miss.

Figure~\ref{fig:gapminder-trade} focuses on a group of four ``trade''-related
variables in the Gapminder dataset detected as probably dependent: ``net trade
balance'', ``total goods traded'', ``exports of goods and services'', and
``imports of goods and services''.
$R^2$ fails to detect a statistically significant dependence
between ``net trade balance'' and the other variables, due to weak linear
correlations and heteroskedastic noise as
shown in the scatter plots (Figure~\ref{fig:trade-scatter}).
From economics, these four variables are causally related by the graphical model
in Figure~\ref{fig:trade-graph}, where the value of a node is a noisy addition
or subtraction of the values of its parents.
Figure~\ref{fig:trade-joint-cmi} illustrates that CrossCat recovers predictive
relationships between these variables: conditioning on ``exports''{=}150 and
``balance''{=}30 (a low probability event according to the left subplot) centers
the posterior predictive distribution of ``imports'' around 120, and decouples
it from ``total goods''.
The posterior CMI curves of ``imports'' and ``total goods'', with and without
the conditions on ``exports'' and ``balance'', formalize this decoupling (right
subplot of Figure~\ref{fig:trade-joint-cmi}).


\section{Related Work}
\label{sec:related}
\vspace{-0.1in}

There is broad acknowledgment that new techniques for dependency detection
beyond linear correlation are required.
Existing approaches for conditional independence testing include the use of
kernel methods
\cite{bach2002,fukumizu2007,zhang2011,sejdinovic2013}, copula functions
\cite{bouezmarni2012,poczos2012,lopez2013}, and characteristic functions
\cite{Su2007}, many of which capture non-linear and multivariate predictive
relationships.
Unlike these methods, however, our approach represents dependence in terms of
conditional mutual information and  is not embedded in a frequentist decision-
theoretic framework.
Our quantity of interest is a full posterior distribution over CMI, as opposed to
a $p$-value to identify when the null hypothesis CMI{=}0 cannot be rejected.
Dependence detection is much less studied in the Bayesian literature;
\cite{holmes2016} use a Polya tree prior to compute a Bayes Factor for the
relative evidence of dependence versus independence.
Their method is used only to quantify evidence for the existence, but not assess
the strength, of a predictive relationship.
The most similar approach to this work was proposed independently in recent work
by \cite{kunihama2016}, who compute a distribution over CMI by estimating the
joint density using an encompassing non-parametric Bayesian prior.
However, the differences are significant.
First, the Monte Carlo estimator in \cite{kunihama2016} is based on resampling
empirical data.
However, real-world databases may be too sparse for resampling data to yield
good estimates, especially for queries given unlikely constraints.
Instead, we use a Monte Carlo estimator by simulating the predictive
distribution.
Second, the prior in \cite{kunihama2016} is a standard Dirichlet process mixture
model, whereas this paper proposes a sparsity-inducing CrossCat prior, which
permits optimized computations for upper bounds of posterior probabilities as
well as simplifying CMI queries with multivariate conditions.


\section{Discussion}
\label{discussion}

This paper has shown it is possible to detect predictive relationships
by integrating probabilistic programming, information theory,
and non-parametric Bayes.
Users specify a broad class of conditional mutual information queries using a
simple SQL-like language, which are answered using a scalable pipeline based on
approximate Bayesian inference.
The underlying approach applies to arbitrary generative population models,
including parametric models and other classes of probabilistic programs
\cite{saad2016}; this work has focused on exploiting the sparsity of CrossCat
model structures to improve scalability for exploratory analysis.
With this foundation, one may extend the technique to domains such as
causal structure learning. The CMI estimator can be used as a
\mbox{conditional-independence} test in a structure discovery algorithm such as
PC \cite{sprites2000}.
It is also possible to use learned CMI probabilities as part of a prior over
directed acyclic graphs in the Bayesian setting.
This paper has focused on detection and preliminary assessment of predictive
relationships; confirmatory analysis and descriptive summarization are left for
future work, and will require an assessment of the robustness of joint density
estimation when random sampling assumptions are violated.
Moreover, new algorithmic insights will be needed to scale the technique to
efficiently detect pairwise dependencies in very high-dimensional databases with
tens of thousands of variables.

\clearpage


\subsubsection*{Acknowledgements}

This research was supported by DARPA (PPAML program, contract number
FA8750-14-2-0004), IARPA (under research contract 2015-15061000003), the Office
of Naval Research (under research contract N000141310333), the Army Research
Office (under agreement number W911NF-13-1-0212), and gifts from Analog Devices
and Google.

\bibliography{writeup}

\begin{thebibliography}{35}
\providecommand{\natexlab}[1]{#1}
\providecommand{\url}[1]{\texttt{#1}}
\expandafter\ifx\csname urlstyle\endcsname\relax
  \providecommand{\doi}[1]{doi: #1}\else
  \providecommand{\doi}{doi: \begingroup \urlstyle{rm}\Url}\fi

\bibitem[Bach and Jordan(2002)]{bach2002}
Francis Bach and Michael Jordan.
\newblock Kernel independent component analysis.
\newblock \emph{Journal of Machine Learning Research}, 3:\penalty0 1--48, 2002.

\bibitem[Bouezmarni et~al.(2012)Bouezmarni, Rombouts, and
  Taamouti]{bouezmarni2012}
Taoufik Bouezmarni, Jeroen~VK Rombouts, and Abderrahim Taamouti.
\newblock Nonparametric copula-based test for conditional independence with
  applications to granger causality.
\newblock \emph{Journal of Business \& Economic Statistics}, 30\penalty0
  (2):\penalty0 275--287, 2012.

\bibitem[Boutilier et~al.(1996)Boutilier, Friedman, Goldszmidt, and
  Koller]{boutilier1996}
Craig Boutilier, Nir Friedman, Moises Goldszmidt, and Daphne Koller.
\newblock Context-specific independence in bayesian networks.
\newblock In \emph{Proceedings of the Twelfth International Conference on
  Uncertainty in Artificial Intelligence}, pages 115--123. Morgan Kaufmann
  Publishers Inc., 1996.

\bibitem[Council(2013)]{data2013}
National~Reseach Council.
\newblock \emph{Frontiers in massive data analysis}.
\newblock The National Academies Press, 2013.

\bibitem[Cover and Thomas(2012)]{cover2012}
T.M. Cover and J.A. Thomas.
\newblock \emph{Elements of Information Theory}.
\newblock Wiley Series in Telecommunications and Signal Processing. Wiley,
  2012.

\bibitem[Edwards(2012)]{edwards2012}
David Edwards.
\newblock \emph{Introduction to graphical modelling}.
\newblock Springer Texts in Statistics. Springer, 2012.

\bibitem[Escobar and West(1995)]{escobar1995}
Michael Escobar and Mike West.
\newblock Bayesian density estimation and inference using mixtures.
\newblock \emph{Journal of the American Statistical Association}, 90\penalty0
  (430):\penalty0 577--588, 1995.

\bibitem[Filippi and Holmes(2016)]{holmes2016}
Sarah Filippi and Chris Holmes.
\newblock A bayesian nonparametric approach to testing for dependence between
  random variables.
\newblock \emph{Bayesian Analysis}, 2016.
\newblock Advance publication.

\bibitem[Fink(1997)]{fink1997}
Daniel Fink.
\newblock A compendium of conjugate priors.
\newblock Technical report, Environmental Statistics Group, Department of
  Biology, Montana State University, 1997.

\bibitem[Fukumizu et~al.(2007)Fukumizu, Gretton, Sun, and
  Sch{\"o}lkopf]{fukumizu2007}
Kenji Fukumizu, Arthur Gretton, Xiaohai Sun, and Bernhard Sch{\"o}lkopf.
\newblock Kernel measures of conditional dependence.
\newblock In \emph{Proceedings of the Twentieth International Conference on
  Neural Information Processing Systems}, pages 489--496. Curran Associates
  Inc., 2007.

\bibitem[Ghosal et~al.(1999)Ghosal, Ghosh, and Ramamoorthi]{ghosal1999}
Subhashis Ghosal, Jayanta Ghosh, and R.V. Ramamoorthi.
\newblock Posterior consistency of dirichlet mixtures in density estimation.
\newblock \emph{The Annals of Statistics}, 27\penalty0 (1):\penalty0 143--158,
  1999.

\bibitem[G{\"o}r{\"u}r and Rasmussen(2010)]{gorur2010}
Dilan G{\"o}r{\"u}r and Carl~Edward Rasmussen.
\newblock Dirichlet process gaussian mixture models: Choice of the base
  distribution.
\newblock \emph{Journal of Computer Science and Technology}, 25\penalty0
  (4):\penalty0 653--664, 2010.

\bibitem[Gretton et~al.(2005)Gretton, Bousquet, Smola, and
  Sch{\"o}lkopf]{gretton2005}
Arthur Gretton, Olivier Bousquet, Alex Smola, and Bernhard Sch{\"o}lkopf.
\newblock Measuring statistical dependence with hilbert-schmidt norms.
\newblock In \emph{Proceedings of the Sixteenth International Conference
  Algorithmic Learning Theory}, pages 63--77. Springer, 2005.

\bibitem[Jain and Neal(2012)]{jain2012}
Sonia Jain and Radford~M Neal.
\newblock A split-merge markov chain monte carlo procedure for the dirichlet
  process mixture model.
\newblock \emph{Journal of Computational and Graphical Statistics}, 13\penalty0
  (1):\penalty0 158--182, 2012.

\bibitem[Kraskov et~al.(2004)Kraskov, St{\"o}gbauer, and
  Grassberger]{kraskov2004}
Alexander Kraskov, Harald St{\"o}gbauer, and Peter Grassberger.
\newblock Estimating mutual information.
\newblock \emph{Physical Review E}, 69\penalty0 (6):\penalty0 066138, 2004.

\bibitem[Kunihama and Dunson(2016)]{kunihama2016}
Tsuyoshi Kunihama and David~B Dunson.
\newblock Nonparametric bayes inference on conditional independence.
\newblock \emph{Biometrika}, 103\penalty0 (1):\penalty0 35--47, 2016.

\bibitem[Lopez-Paz et~al.(2013)Lopez-Paz, Hennig, and Sch{\"o}lkopf]{lopez2013}
David Lopez-Paz, Philipp Hennig, and Bernhard Sch{\"o}lkopf.
\newblock The randomized dependence coefficient.
\newblock In \emph{Proceedings of the Twenty-Sixth International Conference on
  Neural Information Processing Systems}, pages 1--9. Curran Associates Inc.,
  2013.

\bibitem[Mansinghka et~al.(2015)Mansinghka, Tibbetts, Baxter, Shafto, and
  Eaves]{mansinghka2015}
Vikash Mansinghka, Richard Tibbetts, Jay Baxter, Pat Shafto, and Baxter Eaves.
\newblock Bayes{DB}: A probabilistic programming system for querying the
  probable implications of data.
\newblock \emph{CoRR}, abs/1512.05006, 2015.

\bibitem[Mansinghka et~al.(2016)Mansinghka, Shafto, Jonas, Petschulat, Gasner,
  and Tenenbaum]{mansinghka2016}
Vikash Mansinghka, Patrick Shafto, Eric Jonas, Cap Petschulat, Max Gasner, and
  Joshua~B. Tenenbaum.
\newblock {C}ross{C}at: A fully {B}ayesian nonparametric method for analyzing
  heterogeneous, high dimensional data.
\newblock \emph{Journal of Machine Learning Research}, 17\penalty0
  (138):\penalty0 1--49, 2016.

\bibitem[Mardia et~al.(1980)Mardia, Kent, and Bibby]{mardia1980}
Kantilal~Varichand Mardia, John~T Kent, and John~M Bibby.
\newblock \emph{Multivariate analysis}.
\newblock Probability and Mathematical Statistics. Academic Press, 1980.

\bibitem[Moddemeijer(1989)]{moddemeijer1989}
Rudy Moddemeijer.
\newblock On estimation of entropy and mutual information of continuous
  distributions.
\newblock \emph{Signal Processing}, 16\penalty0 (3):\penalty0 233--248, 1989.

\bibitem[Moon et~al.(1995)Moon, Rajagopalan, and Lall]{moon1995}
Young-Il Moon, Balaji Rajagopalan, and Upmanu Lall.
\newblock Estimation of mutual information using kernel density estimators.
\newblock \emph{Physical Review E}, 52\penalty0 (3):\penalty0 2318, 1995.

\bibitem[Obermeyer et~al.(2014)Obermeyer, Glidden, and Jonas]{obermeyer2014}
Fritz Obermeyer, Jonathan Glidden, and Eric Jonas.
\newblock Scaling nonparametric {B}ayesian inference via subsample-annealing.
\newblock In \emph{Proceedings of the Seventeenth International Conference on
  Artificial Intelligence and Statistics}, pages 696--705. JMLR.org, 2014.

\bibitem[Paninski(2003)]{paninski2003}
Liam Paninski.
\newblock Estimation of entropy and mutual information.
\newblock \emph{Neural Computation}, 15\penalty0 (6):\penalty0 1191--1253,
  2003.

\bibitem[P{\'o}czos et~al.(2012)P{\'o}czos, Ghahramani, and
  Schneider]{poczos2012}
Barnab{\'a}s P{\'o}czos, Zoubin Ghahramani, and Jeff Schneider.
\newblock Copula-based kernel dependency measures.
\newblock \emph{CoRR}, abs/1206.4682, 2012.

\bibitem[Robert and Casella(2005)]{robert2005}
Christian Robert and George Casella.
\newblock \emph{Monte Carlo Statistical Methods}.
\newblock Springer Texts in Statistics. Springer, 2005.

\bibitem[Rosling()]{rosling2008}
Hans Rosling.
\newblock Gapminder: Unveiling the beauty of statistics for a fact based world
  view.
\newblock URL \url{https://www.gapminder.org/data/}.

\bibitem[Saad and Mansinghka(2016)]{saad2016}
Feras Saad and Vikash Mansinghka.
\newblock Probabilistic data analysis with probabilistic programming.
\newblock \emph{CoRR}, abs/1608.05347, 2016.

\bibitem[Sejdinovic et~al.(2013)Sejdinovic, Gretton, and
  Bergsma]{sejdinovic2013}
Dino Sejdinovic, Arthur Gretton, and Wicher Bergsma.
\newblock A kernel test for three-variable interactions.
\newblock In \emph{Proceedings of the Twenty-Sixth International Conference on
  Neural Information Processing Systems}, pages 1124--1132. Curran Associates
  Inc., 2013.

\bibitem[Shachter(1998)]{shachter1998}
Ross~D Shachter.
\newblock Bayes-ball: Rational pastime for determining irrelevance and
  requisite information in belief networks and influence diagrams.
\newblock In \emph{Proceedings of the Fourteenth Conference on Uncertainty in
  Artificial Intelligence}, pages 480--487. Morgan Kaufmann Publishers Inc.,
  1998.

\bibitem[Spirtes et~al.(2000)Spirtes, Glymour, and Scheines]{sprites2000}
Peter Spirtes, Clark Glymour, and Richard Scheines.
\newblock \emph{Causation, Prediction, and Search}.
\newblock Adaptive Computation and Machine Learning. MIT Press, 2000.

\bibitem[Su and White(2007)]{Su2007}
Liangjun Su and Halbert White.
\newblock A consistent characteristic function-based test for conditional
  independence.
\newblock \emph{Journal of Econometrics}, 141\penalty0 (2):\penalty0 807--834,
  2007.

\bibitem[Tokdar(2006)]{tokdar2006}
Surya~T Tokdar.
\newblock Posterior consistency of dirichlet location-scale mixture of normals
  in density estimation and regression.
\newblock \emph{The Indian Journal of Statistics}, 68\penalty0 (1):\penalty0
  90--110, 2006.

\bibitem[Whittaker(1990)]{whittaker1990}
Joe Whittaker.
\newblock \emph{Graphical models in applied multivariate statistics}.
\newblock Wiley Series in Probability and Mathematical Statistics. Wiley, 1990.

\bibitem[Zhang et~al.(2011)Zhang, Peters, and Janzing]{zhang2011}
Kun Zhang, Jonas Peters, and Dominik Janzing.
\newblock Kernel-based conditional independence test and application in causal
  discovery.
\newblock In \emph{Proceedings of the Twenty-Seventh Conference on Uncertainty
  in Artificial Intelligence}, pages 804--813. AUAI Press, 2011.

\end{thebibliography}


\clearpage

\appendix

\section{Proof of optimizing a CMI query}
\label{appx:proof-optimizing-cmi}

\begin{proof}[Proof of Lemma~\ref{lemma:crosscat-cmi-sparse}]
\renewcommand{\qedsymbol}{}
We use the product-sum property of the logarithm (line 3) and linearity of
expectation (line 4) to show that CrossCat's variable partition $\bgamma$
induces a factorization of a CMI query.
\begin{flalign*}
&\I_{\G}\left(\x_\mtA{:}\x_\mtB|\hx_\mtC\right)
= \mathbb{E}\left[
    \log \left(
      \frac{\pt(\x_\mtA{:}\x_\mtB|\hx_\mtC)}
        {\pt(\x_\mtA|\hx_\mtC)\pt(\x_\mtB|\hx_\mtC)}
    \right)
  \right]\\
&= \mathbb{E}\left[
    \log \left(
      \prod_{\mtV\in\bgamma}
        \frac{
          \ptv(\x_{\mtA\cap\mtV}, \x_{\mtB\cap\mtV} | \hx_{\mtC\cap\mtV})}
          {\ptv(\x_{\mtA\cap\mtV}|\hx_{\mtC\cap\mtV})
            \ptv(\x_{\mtB\cap\mtV}|\hx_{\mtC\cap\mtV})}
    \right)
  \right]\\
&= \mathbb{E}\left[
    \sum_{\mtV\in\bgamma}
       \log \left(
        \frac{
          \ptv(\x_{\mtA\cap\mtV}, \x_{\mtB\cap\mtV} | \hx_{\mtC\cap\mtV})}
          {\ptv(\x_{\mtA\cap\mtV}|\hx_{\mtC\cap\mtV})
            \ptv(\x_{\mtB\cap\mtV}|\hx_{\mtC\cap\mtV})}
    \right)
  \right]\\
&= \sum_{\mtV\in\bgamma}
    \mathbb{E}\left[
      \log \left(
        \frac{
          \ptv(\x_{\mtA\cap\mtV}, \x_{\mtB\cap\mtV} | \hx_{\mtC\cap\mtV})}
          {\ptv(\x_{\mtA\cap\mtV}|\hx_{\mtC\cap\mtV})
            \ptv(\x_{\mtB\cap\mtV}|\hx_{\mtC\cap\mtV})}
      \right)
    \right]\\
&= \sum_{\mtV\in\bgamma}
  \I_{\G_\mtV} \left(
    \x_{\mtA\cap\mtV}{:}\x_{\mtB\cap\mtV}|
    \hx_{\mtC\cap\mtV}
  \right).
\end{flalign*}
\end{proof}

\vspace{-3cm}

\section{Experimental methods for dependence detection baselines}
\label{appx:experimental-methods}

In this section we outline the methodology used to produce the pairwise $R^2$
and HSIC heatmaps shown in Figures~\ref{subfig:heatmap-pearson-bonferroni} and
\ref{subfig:heatmap-hsic}.
To detect the strength of linear correlation (for $R^2$) and perform a marginal
independence test (for HSIC) given variables $x_i$ and $x_j$ in the Gapminder
dataset, all records in which at least one of these two variables is missing
were dropped.
If the total number of remaining observations was less than three, the null
hypothesis of independence was not rejected due to degeneracy of these methods
at very small sample sizes.
Hypothesis tests were performed at the $\alpha=0.05$ significance level.
To account for multiple testing (a total of $\binom{320}{2}=51040$), a standard
Bonferroni correction was applied to ensure a family-wise error rate of at most
$\alpha$.

We used an open source MATLAB implementation for HSIC (function
\url{hsicTestBoot} from
\url{http://gatsby.ucl.ac.uk/~gretton/indepTestFiles/indep.htm}).
1000 permutations were used to approximate the null distribution, and kernel
sizes were determined using median distances from the dataset.
From Figure~\ref{subfig:heatmap-hsic}, HSIC detects a large number of
statistically significant dependencies.
Figures~\ref{fig:hsic-spurious} and \ref{fig:hsic-common-sense} report spurious
relationships reported as dependent by HSIC but have a low dependence
probability of less than 0.15 according to posterior CMI (Eq~\ref{eq:depprob}),
and common-sense relationships reported as independent HSIC but have a
high dependence probability.

\begin{figure}[H]
\centering
\tiny\bf
\caption{\footnotesize Spurious relationships detected as dependent by HSIC
($p\ll 10^{-6}$) but probably independent
($\mathbb{P}[\I_\G(x_i{:}x_j)>0]<0.15)$ by the MI upper bound.}
\label{fig:hsic-spurious}
\begin{tabular*}{\linewidth}{ll}
variable A & variable B  \\ \midrule
body mass index (men) & privately owned forest (\%) \\
energy use (per capita) & 50+ yrs sex ratio \\
homicide (15-29) & inflation  (annual \%) \\
children out of primary school & male above 60 (\%) \\
income share (fourth 20\%) & suicide rate age 45-59 \\
personal computers & arms exports (US\$) \\
billionaires (per 1M) & mobile subscription (per 100) \\
residential elec. consumption & smear-positive detection (\%) \\
coal consumption (per capita) & age at 1st marriage (women) \\
coal consumption (per capita) & dead kids per woman \\
underweight children & murder (per 100K) \\
female 0-4 years (\%) &  15+ literacy rate (\%) \\
broadband subscribers (\%) & dots new case detection (\%) \\
crude oil prod. (per capita) & TB incidence (per 100K) \\
dependency ratio & people living with hiv  \\ \bottomrule
\end{tabular*}

\includegraphics[width=.9\linewidth]{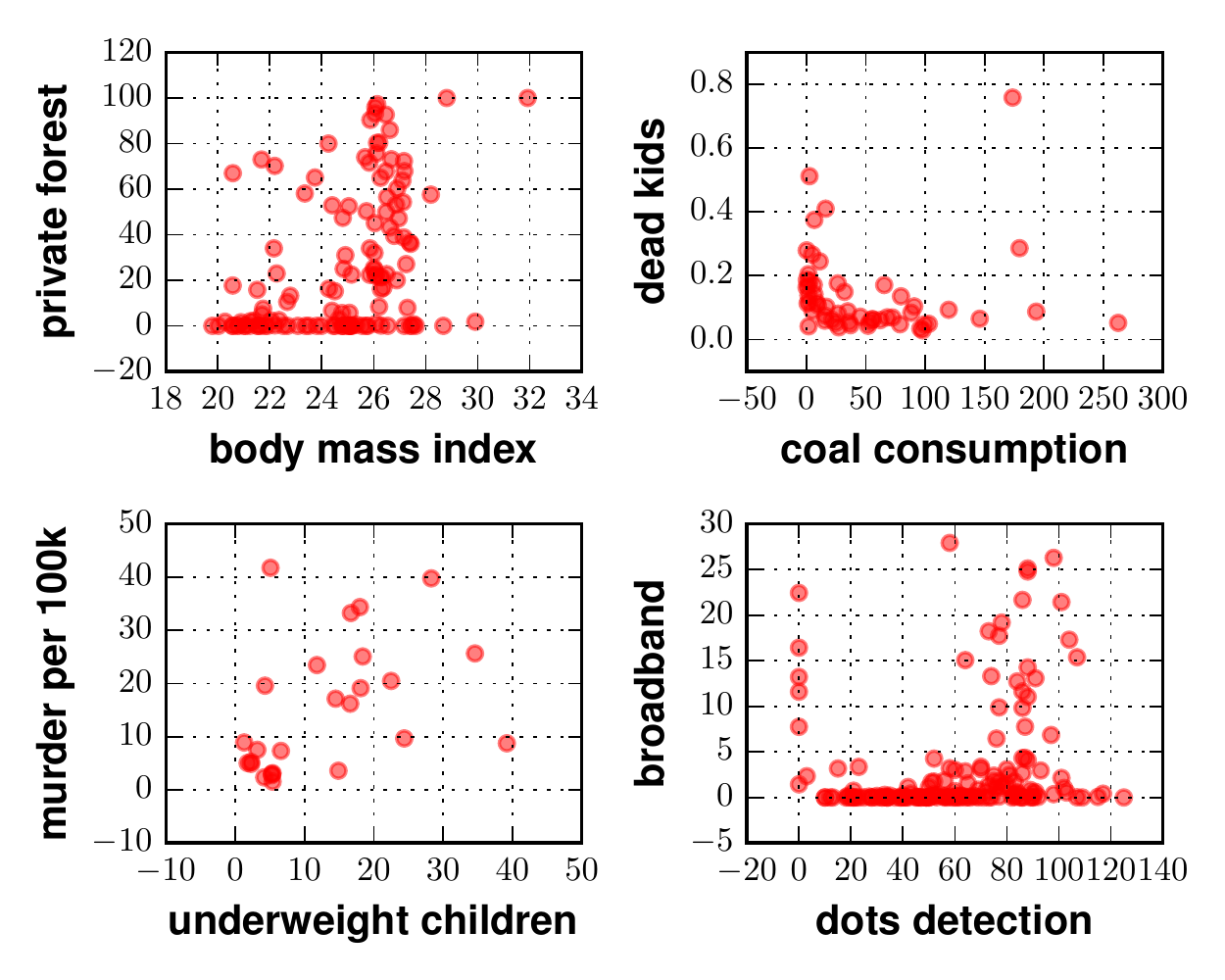}

\rule{\linewidth}{1px}
\medskip

\caption{\footnotesize Common-sense relationships detected as independent by
HSIC ($p\ll 10^{-6}$), but probably dependent
($\mathbb{P}[\I_\G(x_i{:}x_j)>0]>0.85)$ by the MI upper bound.}
\label{fig:hsic-common-sense}

\begin{tabular*}{\linewidth}{ll}
variable A & variable B  \\ \midrule
motor vehicles per 1k & pop urban agglomerations (\%) \\
car mortality (per 100K) & road incidents 45-59 \\
wood removal (cubic m.) & primary forest land (ha)\\
forest products total (US\$) & primary forest land (ha) \\
15-24 yrs sex ratio & 50+ yrs sex ratio \\
gdp per working hr (US\$) & urban population (\%) \\
dots pop. coverage (\%) & dots new case detection (\%) \\
total 15-24 unemp. (\%) & long term unemp. (\%) \\
female self-employed (\%) & female service workers (\%)  \\
female agricult. workers (\%) & total industry workers (\%) \\
female industry workers (\%) & male industry workers (\%) \\
road incidents age 60+ & road incidents age 15-29 \\
suicide rate age 15-29 & suicide rate age 60+ \\ \bottomrule
\end{tabular*}

\includegraphics[width=.9\linewidth]{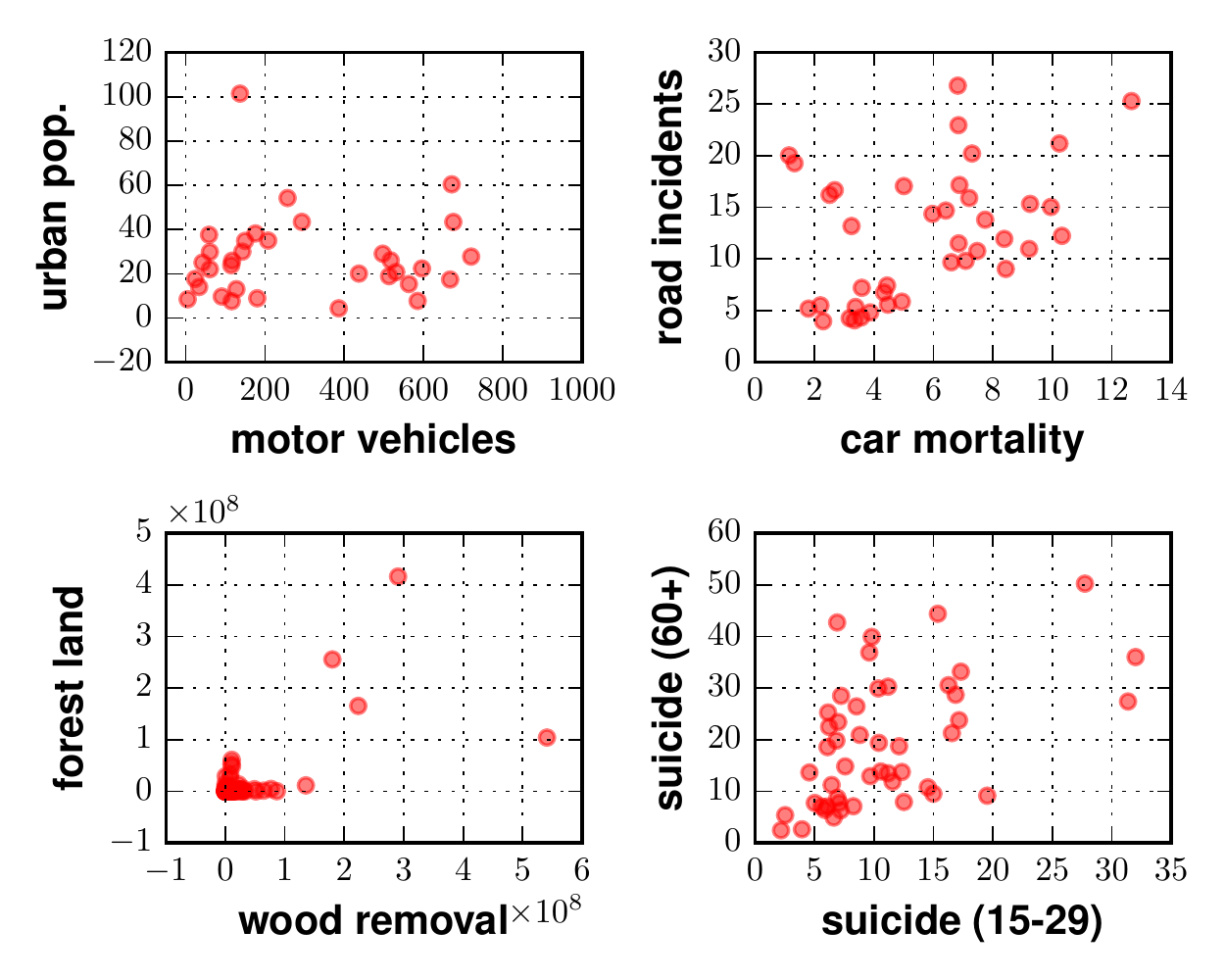}

\end{figure}


\begin{figure*}[t!]

\section{Application to a database of mathematics marks}
\label{appx:applications-marks}

\begin{subtable}[b]{.34\linewidth}
\tiny\bf
\setlength{\tabcolsep}{.6em}
\begin{tabular*}{\linewidth}{|llll@{\extracolsep{\fill}}l|}
\hline
\textbf{mech}
    & \textbf{vectors}
    & \textbf{algebra}
    & \textbf{analysis}
    & \textbf{stats} \\
\hline
77 & 82 & 67 & 67 & 81 \\
23 & 38 & 36 & 48 & 15 \\
63 & 78 & 80 & 70 & 81 \\
55 & 72 & 63 & 70 & 68 \\
\dots & \dots & \dots & \dots & \dots \\
\hline
\end{tabular*}
\captionsetup{width=.9\linewidth}
\subcaption{Database of mathematics marks for 88 students, where rows are
students and columns are exam scores.}
\label{subfig:marks-dataset}
\end{subtable}\hspace{.1cm}%
\begin{subtable}[b]{.34\linewidth}
\tiny\bf
\begin{tabular*}{\linewidth}{|l|@{\extracolsep{\fill}}lllll|}
\hline
\textbf{}
    & \textbf{M}
    & \textbf{V}
    & \textbf{G}
    & \textbf{L}
    & \textbf{S} \\
\hline
\textbf{M} & 1.00  & 0.33  & 0.23 & \red{0.00} & \red{0.03} \\
\textbf{V} & 0.33  & 1.00  & 0.28 & \red{0.08}  & \red{0.02} \\
\textbf{G} & 0.23  & 0.28  & 1.00 & 0.43  & 0.36 \\
\textbf{L} & \red{0.00} & \red{0.08}  & 0.43 & 1.00  & 0.26 \\
\textbf{S} & \red{0.02}  & \red{0.02}  & 0.36 & 0.26  & 1.00 \\ \hline
\end{tabular*}
\captionsetup{width=.95\linewidth}
\subcaption{Partial correlation matrix; red entries indicate statistically
significant conditional independences.}
\label{subfig:marks-pcorr}
\end{subtable}\hspace{.2cm}%
\begin{subfigure}[b]{.3\linewidth}
\begin{tikzpicture}
\coordinate (vectors-cord);
\coordinate[above = .55cm of vectors-cord] (mechanics-cord);
\coordinate[above right = .275cm and 1.5cm of vectors-cord] (algebra-cord);
\coordinate[right = 3cm of vectors-cord] (statistics-cord);
\coordinate[above right = .55cm and 3cm of vectors-cord] (analysis-cord);

\node[circle, draw] (vectors-node) at (vectors-cord) {};
\node[circle, draw] (mechanics-node) at (mechanics-cord) {};
\node[circle, draw] (algebra-node) at (algebra-cord) {};
\node[circle, draw] (statistics-node) at (statistics-cord) {};
\node[circle, draw] (analysis-node) at (analysis-cord) {};

\node[
    below = .1cm of vectors-cord,
    align = center,
] (vectors-text) {
    \tiny\bf vectors (V)
};

\node[
    above = .1cm of mechanics-cord,
    align = center,
] (mechanics-text) {
    \tiny\bf mechanics (M)
};

\node[
    below = .1cm of algebra-cord,
    align = center,
] (algebra-text) {
    \tiny\bf algebra (G)
};

\node[
    below = .1cm of statistics-cord,
    align = center,
] (statistics-text) {
    \tiny\bf statistics (S)
};

\node[
    above = .1cm of analysis-cord,
    align = center,
] (analysis-text) {
    \tiny\bf analysis (L)
};

\draw[-] (vectors-node) -- (mechanics-node);
\draw[-] (algebra-node) -- (vectors-node);
\draw[-] (algebra-node) -- (mechanics-node);
\draw[-] (algebra-node) -- (analysis-node);
\draw[-] (algebra-node) -- (statistics-node);
\draw[-] (statistics-node) -- (analysis-node);

\end{tikzpicture}
\subcaption{Undirected (Gaussian) graphical model implied by the partial
correlation matrix.}
\label{subfig:marks-graphical}
\end{subfigure}

\begin{subfigure}[b]{.45\linewidth}
\includegraphics[width=\linewidth]{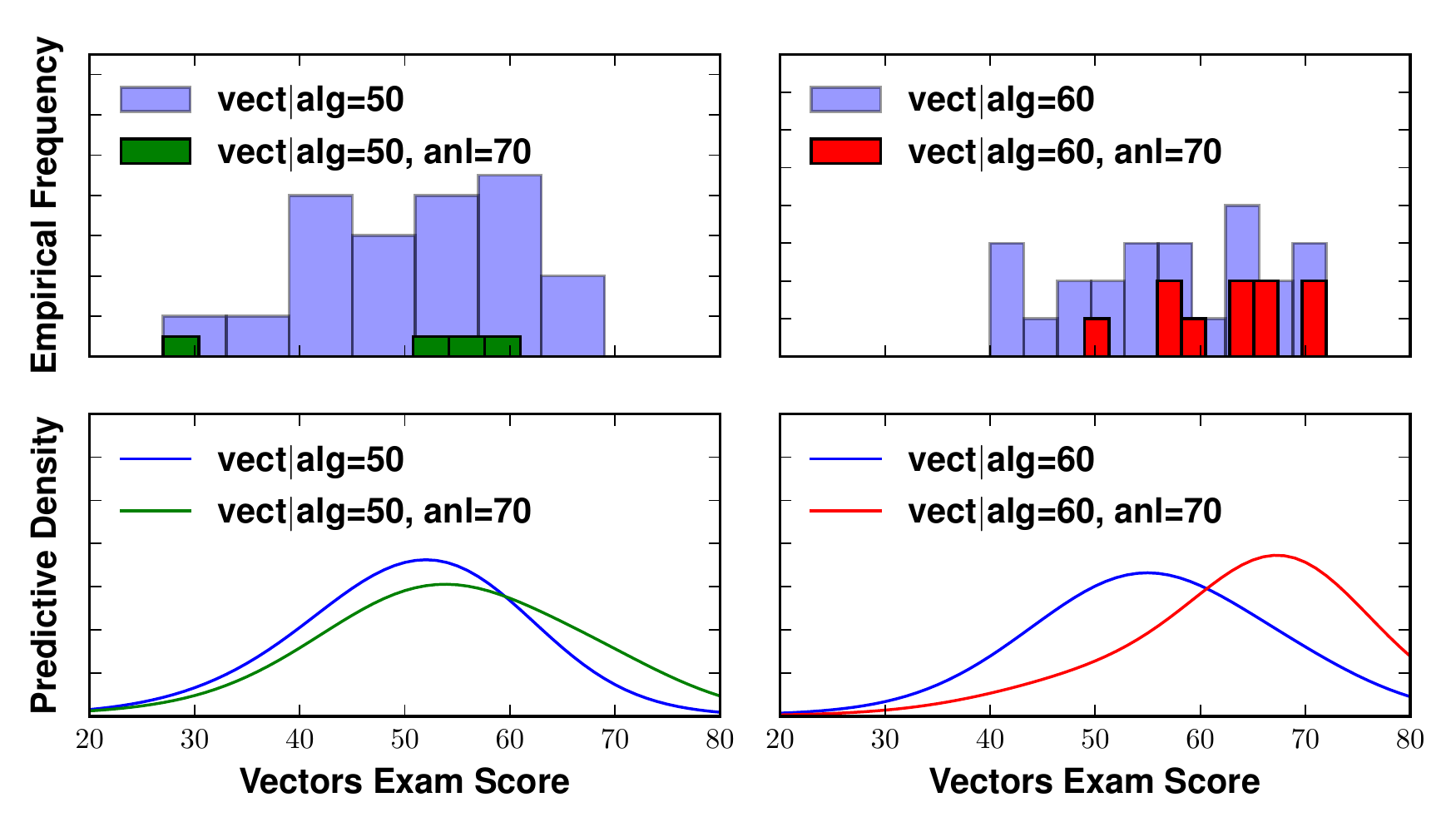}
\captionsetup{width=.95\linewidth}
\subcaption{Histograms from the raw dataset (top); and predictive
distributions from CrossCat (bottom).}
\label{subfig:pdf-vectors-algebra}
\end{subfigure}%
\begin{subfigure}[b]{.55\linewidth}
\includegraphics[width=\linewidth]{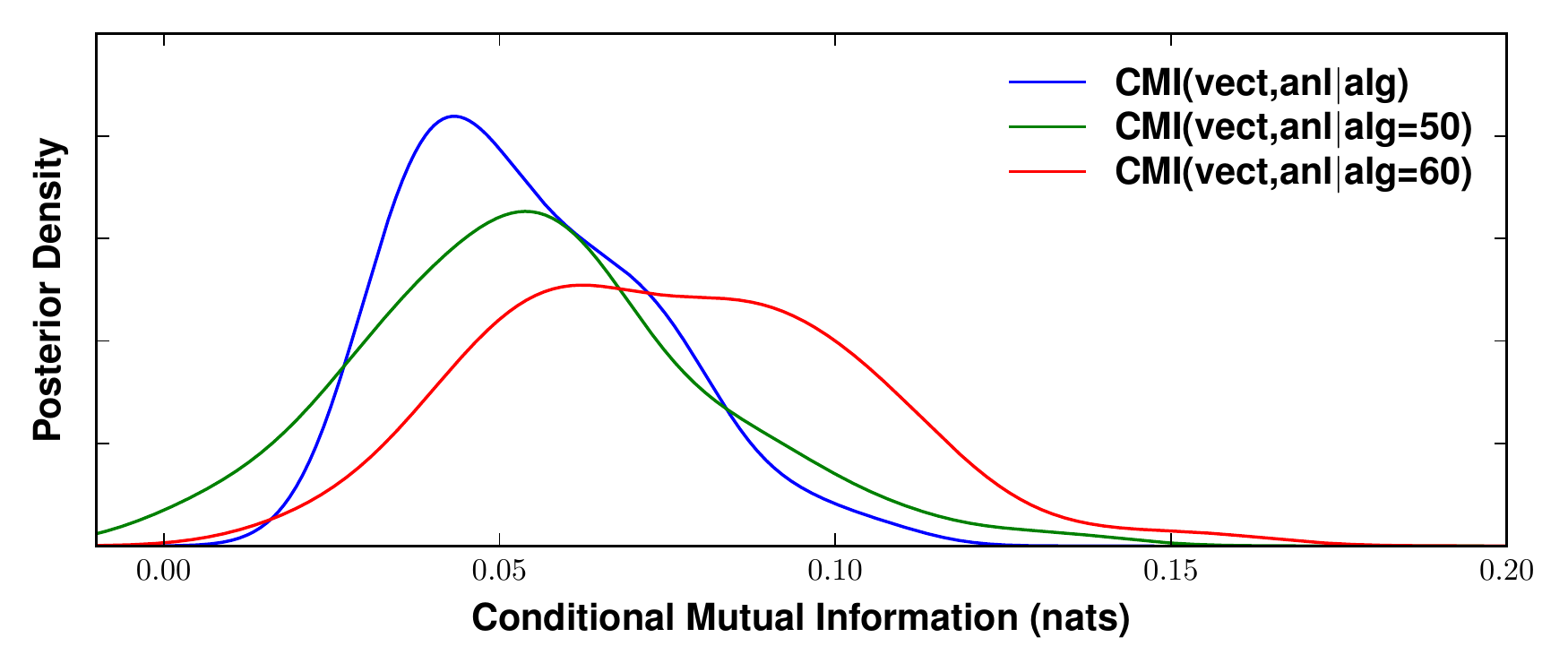}
\captionsetup{width=.95\linewidth}
\subcaption{Posterior distribution of CMI(\texttt{vectors},  \texttt{analysis})
given various conditions of \texttt{algebra} show context-specific
dependence.}
\label{subfig:cmi-vectors-analysis-algebra}
\end{subfigure}

\caption{Using posterior CMI distributions to discover \textit{context-specific}
predictive relationships in the mathematics marks dataset
\cite{mardia1980,whittaker1990,edwards2012} which are missed by partial
correlations.
\textbf{\subref{subfig:marks-dataset}}
The database contains scores of 88 students on five mathematics exams:
\texttt{mechanics}, \texttt{vectors}, \texttt{algebra},
\texttt{analysis}, and \texttt{statistics}.
\textbf{\subref{subfig:marks-pcorr}} Modeling the variables as jointly Gaussian
and computing the partial correlation matrix indicates that (\texttt{mechanics},
\texttt{vectors}) are together conditionally independent of (\texttt{analysis},
\texttt{statistics}), given \texttt{algebra}.
\textbf{\subref{subfig:marks-graphical}} A Gaussian graphical model which
expresses the conditional independences relationships is formed by removing
edges whose incident nodes have statistically-significant partial correlations
of zero. The graph suggests that when predicting the \texttt{vectors} score for
a student whose \texttt{algebra} score is known, further conditioning on the
\texttt{analysis} score provides no additional information. We will critique
this finding, by showing that the predictive strength of \texttt{analysis} on
\texttt{vectors} given \texttt{algebra} varies, depending on the conditioning
value of \texttt{algebra}.
\textbf{\subref{subfig:pdf-vectors-algebra}} The left panel shows that when
\texttt{algebra} = 50, conditioning on \texttt{analysis} = 70 appears to have
little effect on the prediction for \texttt{vectors}. The right panel shows that
when \texttt{algebra} = 60, however, conditioning on \texttt{analysis} = 70
results in a sizeable shift of the posterior mean of \texttt{vectors} from 52 to
just under 70. This shift is consistent with the top right histogram, where
knowing that \texttt{analysis} = 70 eliminates all the \texttt{vectors} scores
in the heavy left tail.
\textbf{\subref{subfig:cmi-vectors-analysis-algebra}}
We formalize this ``context-specific'' dependence by computing the distribution
of the CMI of \texttt{vectors} and \texttt{analysis} under two conditions:
\texttt{algebra} = 50 (green curve), and \texttt{algebra} = 60 (red curve).
The red curve places great probability on higher values of mutual information
than the green curve, which explains the shift in predictive density from
\subref{subfig:pdf-vectors-algebra}. Finally, we observe that the CMI is
weakest when \textit{marginalizing} over all values of \texttt{algebra} (blue
curve), which explains why the partial correlation of \texttt{vectors} and
\texttt{analysis}, which only considers marginal relationships, is near zero.}
\label{fig:marks}
\end{figure*}

\end{document}